\documentclass[11pt]{article}

\usepackage{geometry} 
\usepackage{lineno,hyperref}
\usepackage{amsthm}
\newtheorem{theorem}{Theorem}
\newtheorem{definition}{Definition}
\newtheorem{corollary}{Corollary}

\newtheorem{lemma}{Lemma}

\usepackage{graphicx}
\usepackage{float}
\usepackage{subfigure}
\usepackage{booktabs}
\usepackage{multicol}

\usepackage{amssymb}
\usepackage{bbm}
\usepackage{amsmath}
\usepackage{bm}
\usepackage{newtxtext}

\usepackage{natbib}

\title{Spectral Complexity-scaled Generalization Bound of Complex-valued Neural Networks}

\author{Haowen Chen\thanks{H. Chen is with JD Explore Academy, JD.com Inc. and Faculty of Science, University of Hong Kong. This work was completed when she is an intern at JD Explore Academy. Email: chenhw06@connect.hku.hk.} \and Fengxiang He\thanks{F. He, S. Lei, and D. Tao are with JD Explore Academy, JD.com Inc. Email: hefengxiang@jd.com, leishiye@jd.com, and taodacheng@jd.com.} \and Shiye Lei\footnotemark[2] \and Dacheng Tao\footnotemark[2]}

\date{}

\begin{document}

\maketitle

\begin{abstract}
Complex-valued neural networks (CVNNs) have been widely applied to various fields, especially signal processing and image recognition. However, few works focus on the generalization of CVNNs, albeit it is vital to ensure the performance of CVNNs on unseen data. This paper is the first work that proves a generalization bound for the complex-valued neural network. The bound scales with the spectral complexity, the dominant factor of which is the spectral norm product of weight matrices. Further, our work provides a generalization bound for CVNNs when training data is sequential, which is also affected by the spectral complexity. Theoretically, these bounds are derived via Maurey Sparsification Lemma and Dudley Entropy Integral. Empirically, we conduct experiments by training complex-valued convolutional neural networks on different datasets: MNIST, FashionMNIST, CIFAR-10, CIFAR-100, Tiny ImageNet, and IMDB. Spearman’s rank-order correlation coefficients and the corresponding p values on these datasets give strong proof that the spectral complexity of the network, measured by the weight matrices spectral norm product, has a statistically significant correlation with the generalization ability.  
\end{abstract}

\section{Introduction}
Complex-valued neural networks (CVNNs) are drawing increasing attention in various fields, such as signal processing \citep{goh2005,hirose1996behavior}, voice processing \citep{sawada2003polar}, image reconstruction \citep{cole2020analysis}, etc. CVNN is a kind of neural network whose weight parameters of neurons on each layer are complex numbers. It is known that a complex number consists of a pair of real numbers, the imaginary part and the real part, or the amplitude and the phase. When conducting computations with complex numbers, specific arithmetic rules are applied separately to the imaginary and real parts. Therefore, it is natural to link the CVNNs to the double dimensional real-valued neural networks with less degrees of freedom \citep{hiros,hirose2012}.  

Several recent works endeavored to investigate different properties of CVNNs and build up basic algorithms for the implementation of CVNNs. For example,  \citet{nitta1997,nitta2004,nitta2003inherent,nitta2002redundancy} proved the orthogonality of the decision boundary of complex-valued neurons, addressed the redundancy problem of the parameters of CVNNs, extended the back propagation algorithm to complex numbers, and  \citet{trabe} organized the essential components of  complex-valued deep neural networks, like the complex convolutions, complex batch normalization, and complex weight initialization. Moreover, several empirical works have been done to see the experimental performance of CVNNs.  \citet{hiros} used different neural networks, including CVNNs, to process signals of different coherence, and \citet{nitta1997} found that under the control of the same computational cost, the CVNNs enjoy a higher learning speed than real-valued neural networks.

Albeit past works have presented decent experimental performance of CVNNs, theoretical analysis for their generalization ability is still absent, which motivates us to put effort into deriving a generalization bound of CVNNs. 

This paper is the first work to provide theoretical evidence of the generalization performance of CVNNs. We propose novel upper bounds which positively correlate with the spectral complexity of CVNNs when training on both i.i.d. data and sequential data. These spectral complexity scaled upper bounds suggest a direct correlation between the generalization ability of CVNNs and the spectral norm product of its complex-valued weight matrices. 

For the empirical aspect, we conduct experiments to investigate the influence of spectral complexity on the generalization ability. Specifically, we trained CVNNs by SGD on six standard datasets: CIFAR-10, CIFAR-100, MNIST, FashionMNIST, Tiny ImageNet, and IMDB. Excess risks are collected for analysis. Since the training error is almost zero across all datasets, the excess risk equals the test accuracy and is informative in expressing the generalization ability. Also, since the change of spectral norm product of weight matrices contributes mainly to the change of spectral complexity, it is observed to simulate spectral complexity. The result we present shows that the spectral norm product closely correlates with the excess risk, fully supporting our theoretical analysis. 

The rest of the paper is organized as follows. Section 2 introduces our motivations and reviews some related works. Section 3 recalls the preliminaries of complex-valued neural networks. Section 4 and 5 present the main theorems and experimental results. Finally, section 7 concludes our work.

\section{Motivation and Related Works}

The wide adaptation of complex values in different neural networks is due to its advantages from biological \citep{reich}, computational \citep{nitta1997,danih}, and representational perspectives \citep{arjov,wisdo}. 

From the perspective of biology, \citet{reich} raised that the complex-valued neuronal unit is a more appropriate abstraction to model the activity of neurons in the brain than real-valued ones. 
In order to better process the cortical information, the modeling mechanism must take both the firing rate and spike timing into consideration. To incorporate these two elements into deep neural networks, the amplitude of a complex-valued neuron represents the firing rate, and the phase shall represent the spike timing. When two inputs of an excitatory complex-valued neuron have similar or dissimilar phase information, the magnitude of the net input may increase or decrease according to whether phases are similar or not, which correspond to synchronous and asynchronous situations separately. The incorporation of complex values into deep neural networks helps to construct richer and more versatile representations.

Regarding the computational aspect, \citet{danih} combined LSTM with the idea of Holographic Reduced Representations and used complex values to increase the efficiency of information retrieval. Experiments showed that this method enjoyed a faster learning speed on multiple memorization tasks. 
 \citet{nitta1997} extended the back-propagation algorithm to complex values. The author preserved the basic idea of real-valued back-propagation, but updates were conducted in both real and imaginary parts. Through experiments, it showed that under the same time complexity, the learning speed of complex back-propagation is explicitly faster than the real one when the learning rate is low, i.e., less than 0.5. 

Moreover, complex-valued neural networks also have advantages over real-valued ones in representational ability. \citet{arjov} raised unitary RNN, which used unitary matrices as the weight matrix, to circumvent the well-studied gradient vanishing and gradient exploding issues. The unitary matrix is the generalized form of orthogonal matrices in the complex field, and the absolute value  of its eigenvalues is exactly 1. Compared to the orthogonal matrix, it has richer representations, like applying the discrete Fourier Transformation. \citet{wisdo} further proposed full-capacity unitary RNN, which improved the performance over uRNN.  

Providing all these advantages and applications of CVNNs, it motivates more and more researchers to investigate the properties of complex-valued neural networks and provide the basic framework of implementing CVNNs.  \citet{nitta2004} demonstrated that the decision boundary of a two-layered complex-valued network is orthogonal, and for a three-layered one, the decision boundary is nearly orthogonal. It somehow reflected the computational power of complex value. \citet{trabe} provided building blocks of complex-valued deep neural networks. The author gave complex batch normalization and complex weight initialization strategies, and compared the performances of different activation functions on CIFAR-10, CIFAR-100, and SVHN datasets. 

Although there exist works presenting the empirical results of complex-valued neural networks' generalization performances \citep{nitta1997,hiros}, the theoretical evidence is still absent. Due to this reason, our work is motivated to present the first upper bound for the generalization error of CVNNs. 

Different complexity measures have been raised, such as VC-dimension, Rademacher complexity \citep{mohri}, to derive the generalization bound, and they have been widely applied in different works. For instance, \citet{bartl} proved a margin-based multi-class generalization bound via covering number and Rademacher complexity, and these two tools are also used in our work. However, our work focuses more on complex-valued vector space and provides generalization bounds for complex-valued neural networks when processing regression tasks.

\section{Preliminaries}
This section introduces the complex-valued neural networks (CVNNs) and prepares some notations used in the theoretical analysis.

\subsection{The Model Construction}
Each layer of CVNN consists of several complex-valued neurons described below. The input signals, weight parameters, threshold values, and output signals are all complex numbers in a complex-valued neuron. Assume that for the complex-valued  neuron $n$, it is linked with $m$ numbers of neurons in the previous layer, then the net input to this neuron $n$ is described as follows: 
\begin{equation}
T^{n}_{\text{input}}=\sum_{i=1}^{m}W_{in}X_{in}+H_n.
\end{equation}
Here, $T^{n}_{\text{input}}$ denotes the complex-valued net input of the neuron $n$, $W_{in}$ denotes the weight connecting the neuron $n$ and the neuron $i$ from the previous layer. $X_{in}$ denotes the complex-valued input signal from the neuron $i$ to the neuron $n$, and $H_n$ denotes the threshold value of the neuron $n$. If we denote $Re(T^{n}_{\text{input}})$ and $Im(T^{n}_{\text{input}})$ for the real part and imaginary part of $T^{n}_{\text{input}}$ separately, $\left|T^{n}_{input}\right|$ and $\theta^{n}_{\text{input}}$ for the amplitude and phase of $T^{n}_{\text{input}}$ separately, then the output of the neuron $n$ can be described as follows:

\begin{equation}
T^{n}_{\text{output}}=f_{r}\left(Re(T^{n}_{\text{input}})\right)+f_{i}\left(Re(T^{n}_{\text{input}})\right)
\end{equation}
or 
\begin{equation}
T^{n}_{\text{output}}=e^{if_{p}\left(\theta^{n}_{\text{input}}\right)}f_{a}\left(\left|T^{n}_{input}\right|\right)_{.}
\end{equation}

Equation (2) describes the output derived by applying the activation function separately on the real part and imaginary part, while equation (3) describes the situation when the activation function is applied on the amplitude and phase. In these equations, $f_r$ is the activation function applied on the real part, $f_i$ is the activation function applied on the imaginary part, $f_p$ is the activation function applied on the phase, and $f_a$ is the activation function applied on the amplitude.

\subsection{Complex-valued Activation Functions}
In corresponding to the real-valued activation functions, several forms of complex-valued activation functions are proposed.  

 \citet{arjov} has proposed a $modReLU$ activation function, which preserves the phase information and applies the real-valued ReLU function on the amplitude. The function is described as
\begin{align}
 \operatorname{modReLU}(z)=\operatorname{ReLU}(|z|+b) e^{i \theta_{z}}= \begin{cases}(|z|+b) \frac{z}{|z|} & \text { if }|z|+b \geq 0 \\ 0 & \text { otherwise }\end{cases}.
\end{align}
In this formula, $\left|z\right|$ denotes the amplitude of the complex number $z$, and $b\in \mathbbm{R}$ denotes the threshold for the amplitude of $z$. 

\citet{nitta2002redundancy} raised the following activation function, applying the hyperbolic tangent function to the real part and imaginary part of the input complex number. The function is 
\begin{equation}\sigma(z)=\tanh (Re(z))+i \tanh (Im(z)),
\end{equation}
where $i=\sqrt{-1}$, $\tanh (u) \stackrel{\text { def }}{=}(\exp (u)-\exp (-u)) /(\exp (u)+\exp (-u)), u \in \mathbbm{R}$. 

These two functions represent two main types of complex-valued activation functions. One is applied to the real and imaginary parts, and the other one is applied to the amplitude and phase values. There are other variations of activation functions, such as  $z$ReLU and $\mathbbm{C}$ReLU \citep{guberman2016complex}. These different activation functions have different properties, such as the fulfillment of the Cauchy-Riemann Equations. Therefore, given different situations, activation functions shall be carefully chosen. 

\subsection{Basic Notations and Definitions}

Suppose $S=\{(z_1, y_1),(z_2, y_2), (z_3, y_3),...,(z_n, y_n)   |z_i\in \mathcal{Z} \subset{\mathbb{C}^{d_Z}}, y_i\in \mathcal{Y}\subset{\mathbb{C}^{d_Y}}\}$ is the training sample set, where $y_i$ is the corresponding label of $z_i$,  $d_Z$ and $d_Y$ are the dimensions of the $z$ and $y$ separately.  We define $\mathcal{D}$ to be the distribution that $(z_i, y_i)$ follows.

Assume that the network has $L$ layers, and in the $i$th layer, an $\rho_i$-lipschitz activation function $\sigma_i: \mathbb{C}^{d_{i}}\rightarrow \mathbb{C}^{d_{i}}$ (activation functions such as the $\mathbbm{C}$ReLU function, hyperbolic tangent function, etc. can be used here. Their lipschitz properties are proved in Appendix A) and a weight matrix $A_i \in \mathbb{C}^{d_{i-1}\times d_i}$ are applied to the input matrix passed from the previous layer 
in order. Let $\sigma_i(0)=0$, $\mathcal{A}=(A_1,A_2,...,A_L)$, and $F_{\mathcal{A}}$ to be the function computed by CVNNs: 
\begin{equation}F_{\mathcal{A}}(z):=\sigma_{L}\left(A_{L} \sigma_{L-1}\left(A_{L-1} \cdots \sigma_{1}\left(A_{1} z\right) \right)\right).\end{equation}
The output $F_\mathcal{A}(z) \in \mathbb{C}^{d_L}$   (It's assumed that $d_0=d_Z=d, d_L=d_Y$, and $ W=max\{d_0,d_2,...,d_L\}$). For input data $\{z_1, z_2, ..., z_n\}$, they can form a matrix $Z \in \mathbb{C}^{n\times d}$ by collecting each $z_i$ as the ith row. Therefore, the output of this neural network can be written as $F_\mathcal{A}(Z^T)$, the ith column of which is $F_\mathcal{A}(z_i)$.

To avoid ambiguity, it's necessary to clarify the definition of complex-valued matrix norm. The norm of any complex matrix $[A_{i,j}] \in \mathbb{C}^{d\times k}$ is  defined to be the norm of a corresponding real-valued matrix:
\begin{equation}\left\|\left[A_{i,j}\right]\right\|_p\mathop{=}\limits^{\Delta} \left\|\left[\left|A_{i,j}\right|\right]\right\|_p,\end{equation}
where $A_{i,j}$ denotes the $i,j$th entry of A. In this paper, the  $L_2$ norm is calculated entry-wisely, which means, $L_2$ matrix norm is defined to be the Frobenius norm, i.e., \begin{equation}\left\|A\right\|_{2}\mathop{=}\limits^{\Delta}\sqrt{\sum\limits_{i}\sum\limits_{j}A_{i,j}^2}.\end{equation}

Moreover, $\left\|\cdot\right\|_\sigma$ denotes the spectral norm: \begin{equation}\left\|A\right\|_\sigma:=\sup\limits_{\left\|v\right\|_2=1}\left\|Av\right\|_2=\sqrt{\lambda_{max}(A^{*}A)},\end{equation} where $A^{*}$ denotes the Hermitian transpose of $A$, and $\lambda_{max}$ denotes the largest absolute value of eigenvalues of $A$. Meanwhile,  $\left\|A\right\|_{p,q}$ is defined as: \begin{equation}\left\|A\right\|_{p,q}:\mathop{=}\limits^{\Delta}\left\|\left(\left\|A_{:;1}\right\|_p, \left\|A_{:;2}\right\|_p,...,\left\|A_{:;m}\right\|_p\right)\right\|_q \end{equation}for $A\in \mathbb{C}^{d\times m}$.

To prove the generalization ability, it suffices to derive a high probability bound for the generalization error:
\begin{equation}\mathop{E}\limits_{(z,y)\sim\mathcal{D}}[l(F_\mathcal{A}(z), y)]-\frac{1}{n}\sum\limits_{i=1}^{n}l(F_\mathcal{A}(z_i), y_i),
\end{equation}
where $l(F_\mathcal{A}(z), y):\mathcal{Z}\times\mathcal{Y}\rightarrow \mathbb{R}$ denotes the loss function. It is usually set as
\begin{equation}l(F_\mathcal{A}(z), y)=||F_\mathcal{A}(z)-y||_2.
\end{equation}

Finally, the spectral complexity $R_\mathcal{A}$ of a neural network $F_\mathcal{A}$ is defined as follows: 
\begin{equation}R_{\mathcal{A}}:=\left(\prod_{i=1}^{L} \rho_{i}\left\|A_{i}\right\|_{\sigma}\right)\left(\sum_{i=1}^{L} \frac{\left\|A_{i}^{\top}\right\|_{2,1}^{2 / 3}}{\left\|A_{i}\right\|_{\sigma}^{2 / 3}}\right)^{3 / 2}\end{equation} \citep{bartl}. This complexity measure plays an crucial role in the generalization bound presented next section.

\section{Main Theorems and Proof Sketch}
\subsection{Generalization Bound }
In this section, we present main theorems of this paper. 
\begin{theorem} (i.i.d data)
\label{thm:iid}
Let $S=\{(z_1, y_1),(z_2, y_2), (z_3, y_3),...,(z_n, y_n)\}$ be a sample data set of size n with elements drawn i.i.d from distribution $\mathcal{D}$. 
Given activation functions  $\sigma_i$ ($\sigma_i$ is $\rho_i$-lipschitz and $\sigma_i(0)=0$) and weight matrices $\mathcal{A}=(A_{1}, A_{2},...,A_{L})$ as stated in section 3.3, then with probability at least $1-\delta$, the corresponding complex-valued neural network must satisfy:
\begin{align}
&\mathbb{E}_{(z,y)\sim\mathcal{D}}[l(F_\mathcal{A}(z), y)] - \frac{1}{n}\sum\limits_{i=1}^{n}l(F_\mathcal{A}(z_i), y_i) \nonumber\\
&\leq  \frac{8M}{n^\frac{3}{2}}+\frac{36\left\|Z\right\|_2\sqrt{2\ln(2W)}\ln(n)R_\mathcal{A}}{n}+3M\sqrt{\frac{\ln\frac{2}{\delta}}{2n}},
\end{align}
where $l(F_\mathcal{A}(z), y)=||F_\mathcal{A}(z)-y||_2$ denotes the loss function, and $l(F_\mathcal{A}(z), y)\leq M$ for any $(z,y)$. 
\end{theorem}
It can be observed that there is no explicit occurrence of any combinatorial parameters such as $L$, the number of layers. However, this upper bound depends on $L$ implicitly, as $R_\mathcal{A}$ is formed by each layer's weight matrix norms and the lipschitz constant of activation functions.

The full proof is detailed in Appendix B in detail, while the proof sketch is exhibited in section 4.2. 

\begin{theorem} (sequential data)
\label{thm:non iid}
Consider $S=\{(z_1, y_1),(z_2, y_2), (z_3, y_3),...,(z_n, y_n)\}$ to be a sample data set where $(z_t)_{t\geq 1}$ is a sequence of random data adapted to filtrations $(\mathcal{A}_{t})_{t\geq 1}$. Given activation functions  $\sigma_i$ ($\sigma_i$ is $\rho_i$-lipschitz and $\sigma_i(0)=0$) and weight matrices $\mathcal{A}=(A_{1}, A_{2},...,A_{L})$ as stated in section 3.3, then with probability at least $1-\delta$, the corresponding complex-valued neural network must satisfy: 
\begin{align}
&\frac{1}{n} \sum_{t=1}^{n}\left(\mathbb{E}\left[l\left(z_{t}, y_{t}\right) \mid \mathcal{A}_{t-1}\right]-l\left(z_{t}, y_{t}\right)\right) \nonumber\\
&\leq \frac{8M}{n}+\frac{24\left\|Z\right\|_2\sqrt{2\ln(2W)}\ln(n)R_\mathcal{A}}{n}+M\sqrt{\frac{\ln\frac{2}{\delta}}{2n}},\end{align}
where $l(z, y)=||F_\mathcal{A}(z)-y||_2 $  denotes the loss function, and $l(z, y)\leq M$ for any $(z,y)$. 
\end{theorem}
Theorem \ref{thm:non iid} illustrates the generalization ability of complex-valued neural networks when dealing with sequential data. The proof sketch of this theorem is omitted in the main text because there exists some overlapping with Theorem \ref{thm:iid}, but the full proof is shown in Appendix D. In the Appendix, we also present definitions of sequential Rademacher complexity, sequential covering number, and sequential Dudley Entropy Integral, which were put forward in the work of  \citet{rakhl}.

\subsection{Proof Sketch}
 In this section, we provide the proof sketch of Theorem \ref{thm:iid} via the following lemmas. 
 
 The proof is presented in three steps: \textbf{I)} obtain an upper bound for the covering number: $\mathcal{N}(\{ZA: A\in \mathbb{C}^{d\times m}, \left\|A\right\|_{q,s}\leq a, \epsilon\})$
 as Lemma \ref{linear cov} states.  \textbf{II)} start with a single layer, and apply the induction method to derive an upper bound for the covering number of the whole network. The result is illustrated in lemma \ref{network cov}. \textbf{III)} The proof of Theorem \ref{thm:iid} is ended by substituting the upper bound of Rademacher complexity, which is derived via Dudley Entropy Integral and the above covering number bound, for $\hat{\mathfrak{R}}_S(\mathcal{G})$ in Theorem \ref{thm:non iid}
 
Before further illustration of the proof, we firstly state Theorem \ref{thm:non iid}, which will be a crucial tool in step \textbf{III}. This theorem derives the generalization bound for regression in the case of $L_p$ loss function through Rademacher complexity. We recall this theorem presented by  \citet{mohri}. 

\begin{theorem} [\cite{mohri}]
\label{mohri}
Let $L:\mathcal{Z}\times \mathcal{Y} \rightarrow \mathop{R}$ be an $L_p$ loss function bounded by $M>0$, $\mathcal{F}$  be the hypothesis set, family $\mathcal{G}=\{(x, y) \mapsto l(F_{\mathcal{A}}(x), y):  F_{\mathcal{A}}\in \mathcal{F}\}$, then for any $\delta$, with probability at least $1-\delta$, the following inequality holds: 
\begin{equation}\underset{(x, y) \sim \mathcal{D}}{\mathbb{E}}[l(x, y)] \leq \frac{1}{m} \sum_{i=1}^{m} l\left(x_{i}, y_{i}\right)+2 \hat{\mathfrak{R}}_S(\mathcal{G})+3 M \sqrt{\frac{\log \frac{2}{\delta}}{2 m}}
\end{equation}
where $\hat{\mathfrak{R}}_{S}(\mathcal{G})$ denotes the empirical Rademacher complexity of family $\mathcal{G}.$
\end{theorem}

Obviously, to bound the generalization error, it suffices to derive an upper bound for the Rademacher complexity of the loss function family $\mathcal{G}=\{(x, y) \mapsto l(F_{\mathcal{A}}(x), y):  F_{\mathcal{A}}\in \mathcal{F}\}$, which is realized through the first and second steps. 

\textbf{Step I} \quad In this step, we aim at obtaining a matrix covering for the set of matrix products $ZA$ ($Z $ represents the data matrix passed to the present layer, and $A$ will be instantiated as the weight matrix) under $L_2$ norm. 
\begin{lemma}
\label{linear cov}
 $(p,q),(r,s)$ are two conjugate exponents with $p\leq 2$. Let  $a, b, \epsilon$ be three positive real numbers, and d,m be two positive integers. Let a constraint on the norm of $Z$ be imposed such that $||Z||_p\leq b$. Therefore, we have
 \begin{equation}\ln \mathcal{N}\left(\left\{ZA: A \in \mathbb{C}^{d \times m},\|A\|_{q, s} \leq a\right\}, \epsilon,\|\cdot\|_{2}\right) \leq\left\lceil\frac{a^{2} b^{2} m^{2 / r}}{\epsilon^{2}}\right\rceil \ln (4 d m).
 \end{equation}
\end{lemma}

Basically, the proof of lemma \ref{linear cov} is based on the Maurey sparcification lemma. This lemma inspires us to cover the targeting set by a sparsifying convex hull of complex-valued matrices, which is constructed by the product of the re-scaled data matrix $Z$ \citep{zhang} and  some "standard matrices", such as $\mathbf{e}_i\mathbf{e}_{j}^{T}$.  Moreover, to prove Theorem \ref{thm:iid},  constraints are imposed to $||A||_{2,1}$ (i.e. q=2,s=1), instead of $||A||_{2}$, which helps to avoid any occurrence of combinatorial numbers such as $L$ and $W$ outside of the log term in the upper bound \citep{bartl}. 

\textbf{Step II}\quad As we have obtained the matrix covering upper bound in  \textbf{Step I}, we need to extend the idea to prove the whole network covering number upper bound. The proof of Lemma \ref{network cov} will rely on induction and lemma \ref{linear cov}. 
\begin{lemma}
\label{network cov}
$(\sigma_1,\sigma_2,\sigma_3,...,\sigma_L)$ are fixed activation functions with each $\sigma_i$ being $\rho_i-lipschitz$. Denotes the spectral norm bound of matrix $A_i$ to be $s_i$, and the matrix (2,1) norm bound to be $b_i$ ($i\in \{1,2,...,L\}$).Given $Z$ to be the fixed data matrix, where $Z \in \mathbbm{C}^{n\times d}$, and each row denotes a group of data points, then for any $\epsilon$, we have 
\begin{equation}\ln \mathcal{N}\left(\mathcal{F}, \epsilon,\|\cdot\|_{2}\right) \leq \frac{\|Z\|_{2}^{2} \ln \left(4 W^{2}\right)}{\epsilon^{2}}\left(\prod_{j=1}^{L} s_{j}^{2} \rho_{j}^{2}\right)\left(\sum_{i=1}^{L}\left(\frac{b_{i}}{s_{i}}\right)^{2 / 3}\right)^{3},
\end{equation}
where $\mathcal{F}:=\left\{F_{\mathcal{A}}\left(Z^{T}\right): \mathcal{A}=\left(A_{1}, \ldots, A_{L}\right),\left\|A_{i}\right\|_{\sigma} \leq s_{i},\left\|A_{i}^{\top}\right\|_{2,1} \leq b_{i}\right\}$ is the family of outputs generated by feasible choices of complex-valued neural networks $\mathcal{F}_\mathcal{A}$, and $W$ denotes the maximal dimension of $\{d_0,d_1,,,,d_L\}.$
\end{lemma}
 In general, we separate the proof of this lemma in two parts. The first part is to find out the relationship between the whole network upper bound and the matrix covering bounds of the previous $L$ layers, which is handled in Appendix B.3 Lemma 6. The second part is to combining Lemma \ref{linear cov} and Lemma 6, which together gives Lemma \ref{network cov} through the induction technique.
 
 \textbf{Step III} \quad Since only deriving a bound for the covering number of the whole network $\mathcal{N}(\mathcal{F}, \epsilon, \left\|\cdot\right\|_2)$ is not enough, we still have to derive an upper bound of empirical Rademacher complexity of the loss function family ($ \hat{\mathfrak{R}}_{S}(\mathcal{G})$). It's natural to think of connecting these two concept via Dudley Entropy Integral.  However, a little preparation work need to be done to satisfy the conditions of using standard Dudley Entropy Integral. 
 
 As the standard Dudley Entropy Integral only illustrates the relation between $\mathcal{N}(\mathcal{G}, \epsilon, \left\|\cdot\right\|_2)$ and $\hat{\mathfrak{R}}_{S}(\mathcal{G})$. Hence, to begin with, Lemma \ref{loss family} upper bounds $\mathcal{N}(\mathcal{G}, \epsilon, \left\|\cdot\right\|_2)$ by $\mathcal{N}(\mathcal{F}, \epsilon, \left\|\cdot\right\|_2)$

\begin{lemma}
\label{loss family}
Given family $\mathcal{F}:=\left\{F_{\mathcal{A}}(Z): \mathcal{A} \in \mathcal{B}_{1} \times \cdots \times \mathcal{B}_{L}\right\}$, and family $\mathcal{G}:=\left\{(z, y) \mapsto l\left(F_{\mathcal{A}}(z), y\right): F_{\mathcal{A}} \in \mathcal{F}\right\}$, then the covering number of these two families satisfy
\begin{equation}
  \mathcal{N}(\mathcal{F},\epsilon,\left\|\cdot\right\|_2)\geq\mathcal{N}(\mathcal{G},\epsilon,\left\|\cdot\right\|_2),
\end{equation}
if we let $l\left(F_{\mathcal{A}}(z), y\right)=\left\|\mathcal{F}_\mathcal{A}(z)-y\right\|_2.$
\end{lemma}
Moreover, since the range of the loss function we adapted does not lies in $[0,1]$, Lemma \ref{scaling} investigates the covering number after rescaling. 
\begin{lemma}
\label{scaling}
If a coefficient, say $\alpha>0$, is multiplied to the targeting set $\mathcal{G}$ and distant constant $\epsilon$, then the covering number shall remain unchanged, i.e.,
  \begin{equation}
   \mathcal{N}(\mathcal{G},\epsilon,\left\|\cdot\right\|_2)=\mathcal{N}(\alpha\mathcal{G},\alpha\cdot\epsilon,\left\|\cdot\right\|_2).   
  \end{equation}
  Here $\alpha \mathcal{G}$ represents a set which is obtained by scaling $\alpha$ to each element of $\mathcal{G}$.
 \end{lemma}
With the help of Lemma \ref{loss family} and Lemma \ref{scaling},  the Rademacher complexity of $\mathcal{G}$ can be bounded through Dudley Entropy Integral. We prove Theorem \ref{thm:iid} through directly substituting $ \hat{\mathfrak{R}}_{S}(\mathcal{G})$ in Theorem \ref{thm:non iid} with the value of upper bound we obtained. The detained proof is exhibited in Appendix B.4.

\section{Experimental Results}
In this section, we present experimental results of training complex-valued convolutional neural networks by SGD on six different datasets: MNIST, FashionMNIST, CIFAR-10, CIFAR-100, Tiny ImageNet, and IMDB.

Before presenting the experimental result, a recap of the two upper bounds we derived in Theorem \ref{thm:iid} and Theorem \ref{thm:non iid} shall be given. In both i.i.d. data case and sequential data case, we show that the upper bound we derive scales with the spectral complexity of this complex-valued neural network: 
\begin{equation}R_{\mathcal{A}}:=\left(\prod_{i=1}^{L} \rho_{i}\left\|A_{i}\right\|_{\sigma}\right)\left(\sum_{i=1}^{L} \frac{\left\|A_{i}^{\top}\right\|_{2,1}^{2 / 3}}{\left\|A_{i}\right\|_{\sigma}^{2 / 3}}\right)^{3 / 2}.
\end{equation}
The formula for the spectral norm $R_{\mathcal{A}}$ consists of two parts: the lipschitz constant of this neural network $(\prod_{i=1}^{L} \rho_{i}\left\|A_{i}\right\|_{\sigma})$ and another factor related to the sum of quotients of weight matrix norms $((\sum_{i=1}^{L} \frac{\left\|A_{i}^{\top}\right\|_{2,1}^{2 / 3}}{\left\|A_{i}\right\|_{\sigma}^{2 / 3}})^{3 / 2})$. As in the training process, the part which dominates the change of $R_{\mathcal{A}}$ is the first part, the lipschitz constant of the neural network, and the lipschitz constants of activation functions ($\rho_i$) remain unchanged. Therefore, we use the change of the spectral norm product ( $\prod_{i=1}^{L} \left\|A_{i}\right\|_{\sigma} $) to simulate the changing trend of $R_{\mathcal{A}}$.

\begin{figure}[t]
\centering
\subfigure[]{
\begin{minipage}[b]{0.3\textwidth}
    		\includegraphics[width=0.99\columnwidth]{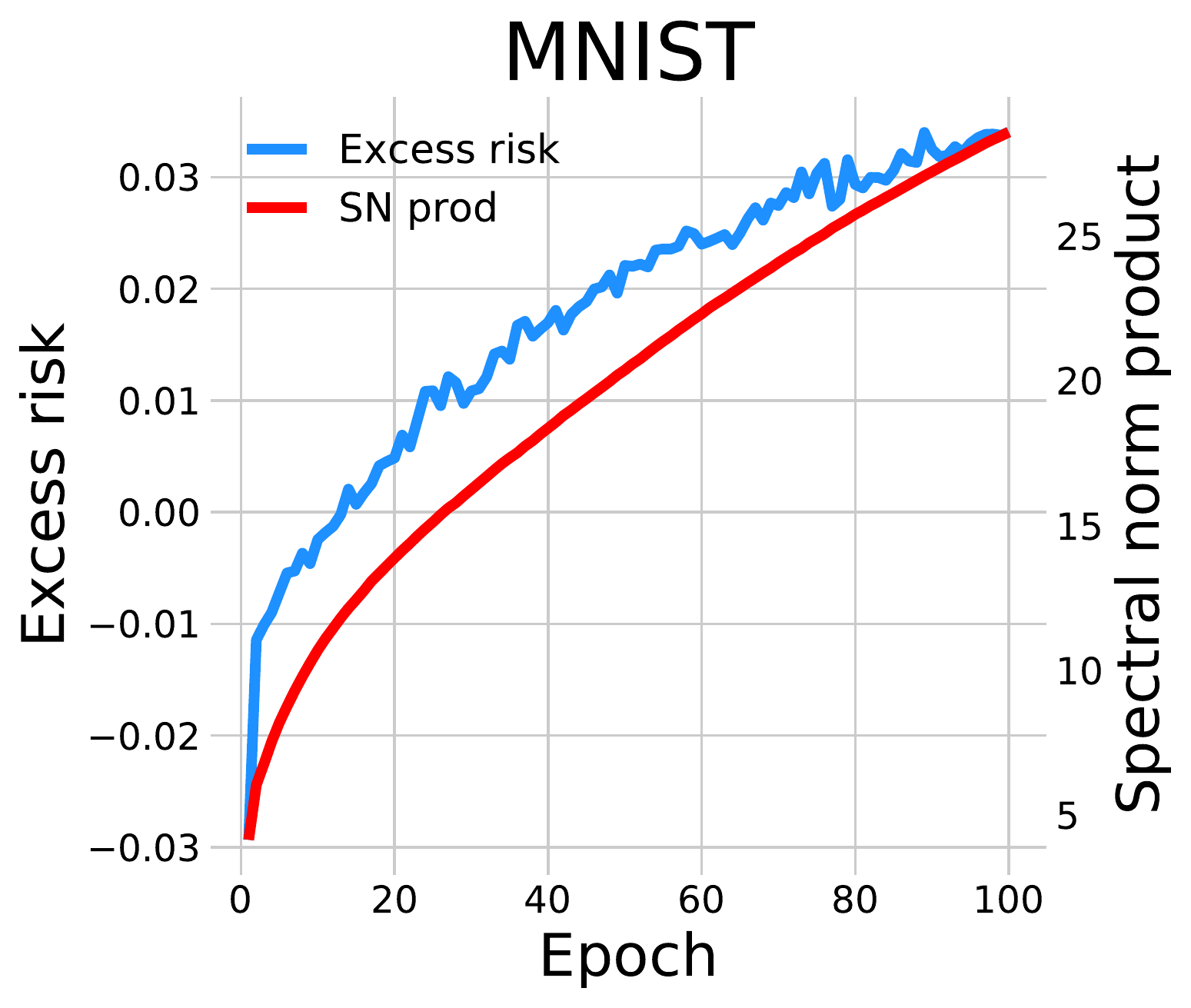}
    		\end{minipage}
		\label{figure:mnist}   
    	}
\subfigure[]{
\begin{minipage}[b]{0.3\textwidth}
    		\includegraphics[width=0.99\columnwidth]{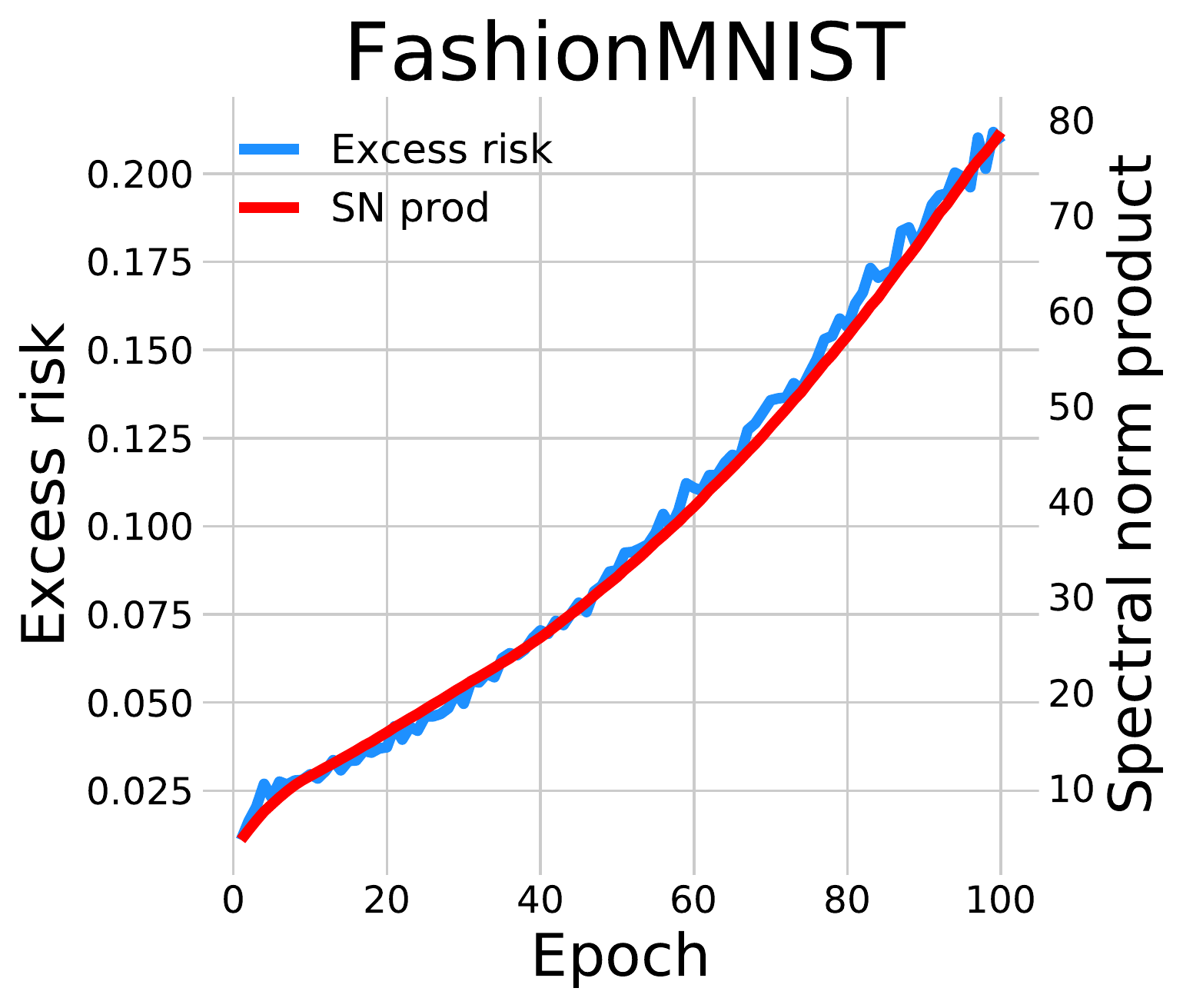}
    		\end{minipage}
		\label{figure:fashionmnist}   
    	}
\subfigure[]{
\begin{minipage}[b]{0.3\textwidth}
    		\includegraphics[width=0.99\columnwidth]{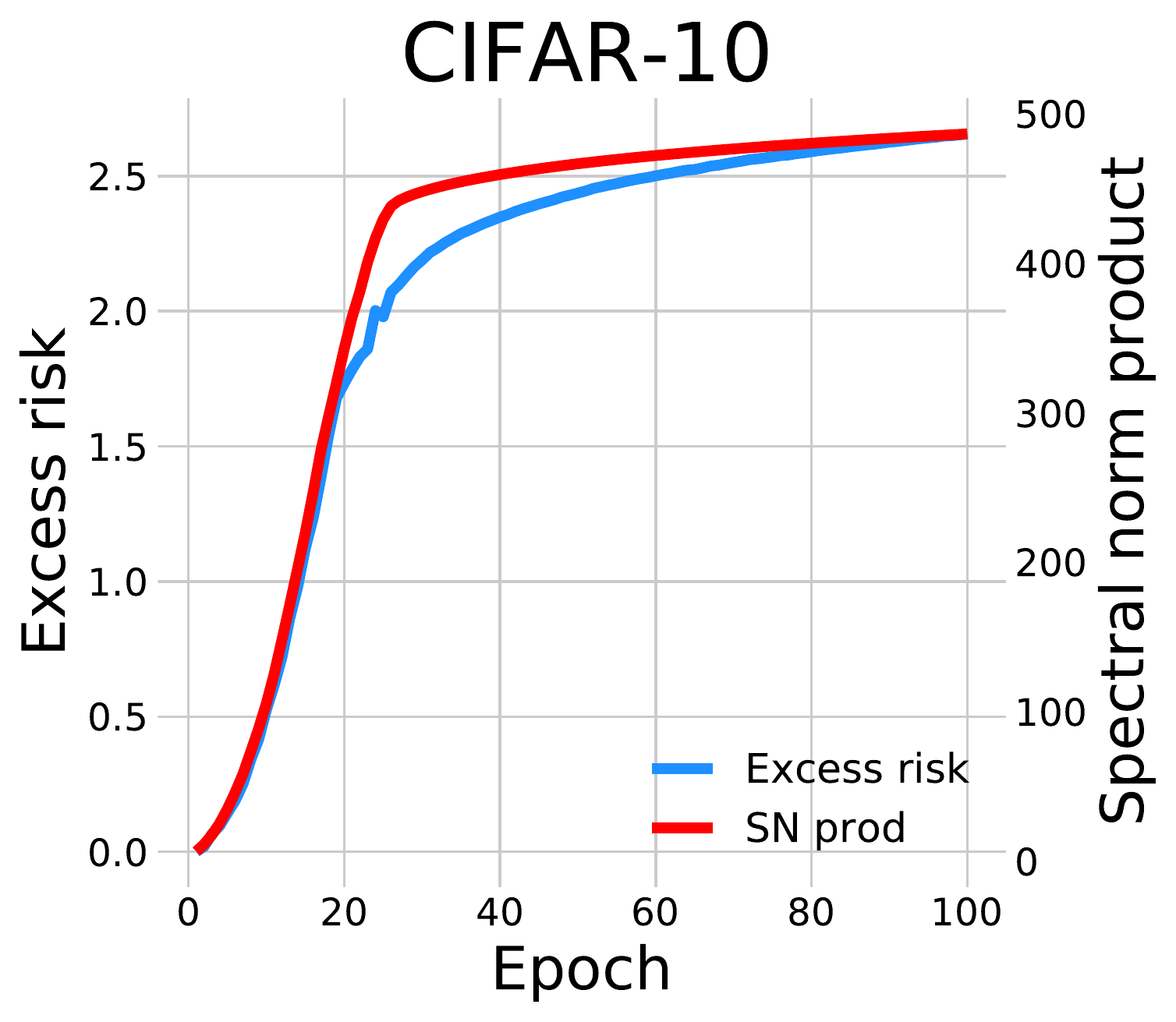}
    		\end{minipage}
		\label{figure:cifar10}   
    	}
\vfill
\subfigure[]{
\begin{minipage}[b]{0.3\textwidth}
    		\includegraphics[width=0.99\columnwidth]{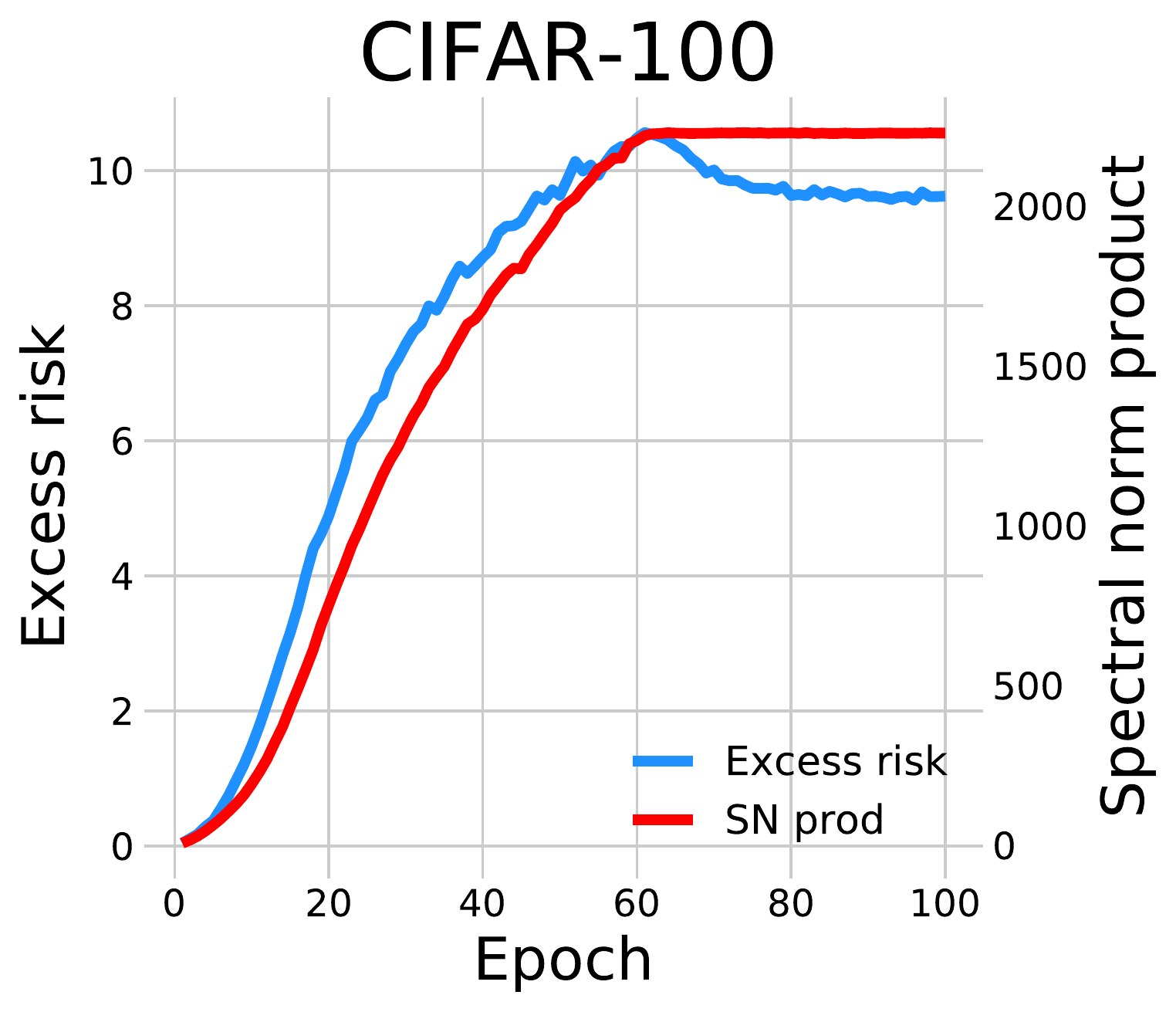}
    		\end{minipage}
		\label{figure:cifar100}   
    	}
\subfigure[]{
\begin{minipage}[b]{0.3\textwidth}
    		\includegraphics[width=0.99\columnwidth]{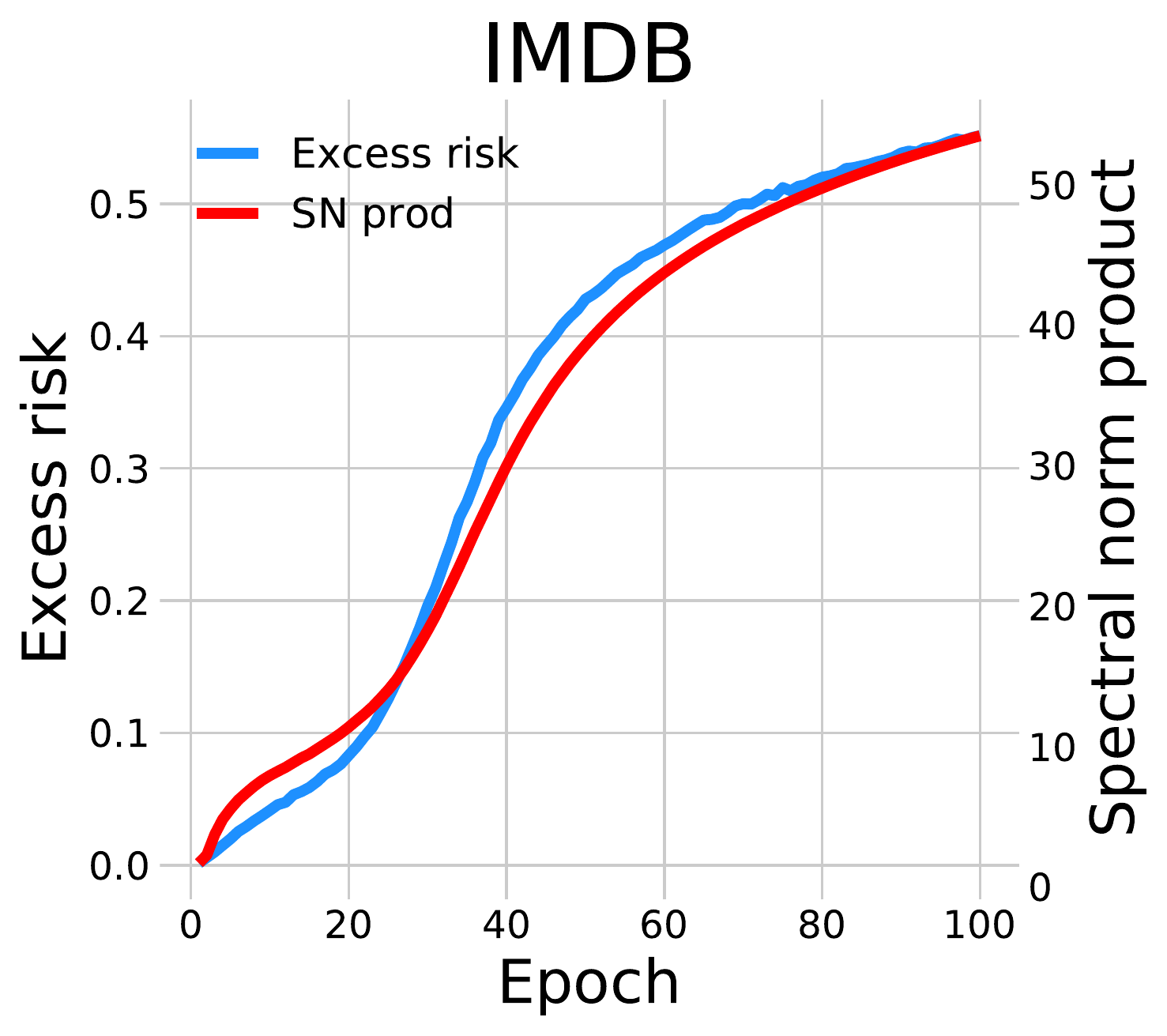}
    		\end{minipage}
		\label{figure:imdb}   
    	}
\subfigure[]{
\begin{minipage}[b]{0.3\textwidth}
    		\includegraphics[width=0.99\columnwidth]{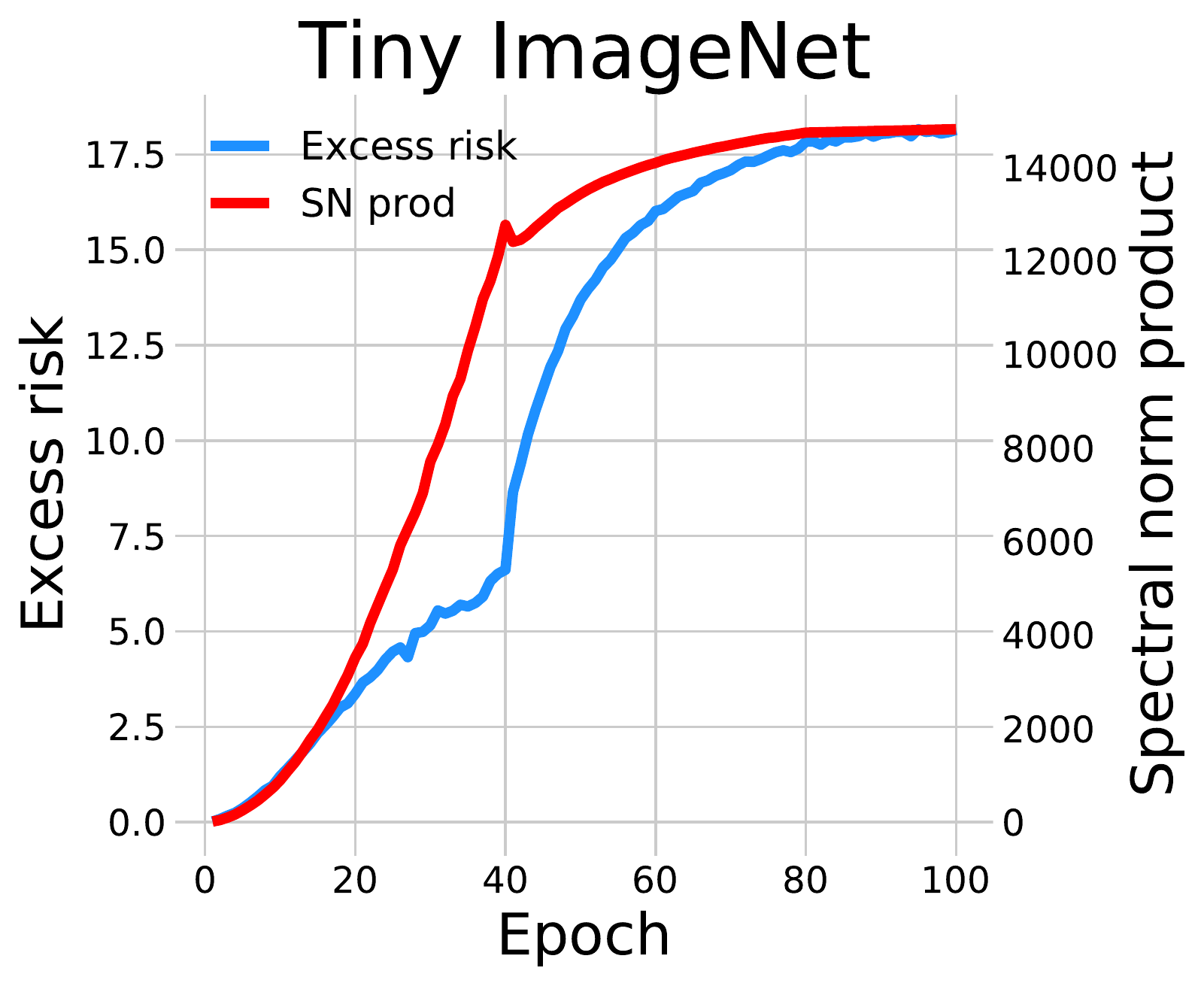}
    		\end{minipage}
		\label{figure:tinyimagenet}   
    	}
\caption{Plots of excess risk and spectral norm product (SN prod) as functions of epoch. The right y-axis denotes the spectral norm product, and the left y-axis denotes the excess risk.}
\label{figure:experiment}
\end{figure}

\subsection{Spectral Norm of the Weight Matrix}
We first show how to calculate the spectral norm of the complex weight matrix in each convolutional layer. 

Considering the complex-valued kernel $W=X+Yi$ in each layer, where $X$ and $Y$ are two real-valued kernels. Since each convolutional kernel is corresponding to a linear transformation weight matrix \citep{guo2019frobenius}, therefore we can derive the real-valued weight matrices of kernels $X$ and $Y$, denoted by $C$ and $D$. Hence, the complex-valued weight matrix $A$ of each layer can be expressed as $C+Di$. Then, by definition, the spectral norm of the complex-valued matrix $A$ is: 
\begin{align}
\|A\|_{\sigma}:&=\sup _{\|v\|_{2}=1}\|A v\|_{2}=\sqrt{\lambda_{\max }\left(A^{*} A\right)} \nonumber\\
&=\sqrt{\lambda_{\max }\left(C^{T}C+D^{T}D+(C^{T}D-CD^{T})i\right)}.
\end{align}
Here, since $A^{*}A$ is a Hermitian matrix, therefore it will only have real eigenvalues.

\subsection{Results}

The architectures of the complex-valued neural networks we used are described in Appendix E.2.
The datasets we used are MNIST, FashionMNIST, CIFAR-10, CIFAR-100, Tiny ImageNet, and IMDB.  Descriptions for these datasets are presented in Appendix E.1. We trained the CVNNs by SGD on MNIST, FashionMNIST, CIFAR-10, CIFAR-100 and Tiny ImageNet to investigate the generalization bound derived in Theorem \ref{thm:iid}, and we trained the CVNN on IMDB to investigate the generalization bound derived in Theorem \ref{thm:non iid} when training data are sequential. Results are shown in Figure \ref{figure:experiment}.

In Figure \ref{figure:experiment}, the plot of excess risk and spectral norm product as functions of epoch is illustrated. Additionally, we perform Spearman rank-order correlation test on all the excess risks and spectral norm products of MNIST, FashionMNIST, CIFAR-10, CIFAR-100, IMDB and Tiny ImageNet.  The Spearman’s rank-order correlation coefficients (sccs) and p values show that the correlation between the spectral norm product and the generalization ability is statistically significant ($p<0.005$\footnote{
The definition of “statistically significant” has various versions, such as $p < 0.05$ and $p < 0.01$. This paper
uses a more rigorous one ($p < 0.005$).}),  as Table \ref{tabl} demonstrates. The result strongly supports our theoretical discoveries. 

\begin{table}[t]
	
		\centering
	\caption{SCC and p values of the spectral norm product and excess risk}
	\label{tabl}%
	\begin{tabular}{cccccccccccc}
	     \toprule 
		 
	     \multicolumn{2}{c}{CIFAR-10}
		&\multicolumn{2}{c}{CIFAR-100} 
		&\multicolumn{2}{c}{MNIST}
	\\
        \cmidrule(lr){1-2} 
        \cmidrule(lr){3-4} 
        \cmidrule(lr){5-6} 
        
		 SCC&p&SCC&p&SCC&p\\
	    \cmidrule(lr){1-2} 
        \cmidrule(lr){3-4} 
        \cmidrule(lr){5-6} 
       0.99&$3.703\times 10^{-228}$&0.80&$4.124\times10^{-23}$&0.99&$9.044\times10^{-142}$\\
        \toprule
        \multicolumn{2}{c}{IMDB} 	
		&\multicolumn{2}{c}{FashionMNIST}
		&\multicolumn{2}{c}{Tiny ImageNet}\\
		\cmidrule(lr){1-2} 
        \cmidrule(lr){3-4} 
        \cmidrule(lr){5-6} 
        SCC&p&SCC&p&SCC&p\\
          \cmidrule(lr){1-2} 
        \cmidrule(lr){3-4} 
        \cmidrule(lr){5-6} 
       0.99&$6.118\times 10^{-194}$&0.99&$3.703\times10^{-142}$&0.99&$4.060\times10^{-125}$\\
	\bottomrule
	\end{tabular}%

\end{table}%

\section{Conclusions}
This work presents two complex-valued neural network generalization bounds under i.i.d. data case and sequential data case. These bounds scale with the spectral complexity,  which contains the spectral norm product of weight matrices as a factor, and are proved from
both theoretical and empirical aspects in this paper. We hope that our work can provide theoretical evidence for the generalization ability of complex-valued neural networks, and stimulate more investigation on other properties of complex-valued neural networks.

\bibliography{mybibfile}
\appendix
\section*{Appendices}
\addcontentsline{toc}{section}{Appendices}
\renewcommand{\thesubsection}{\Alph{subsection}}
\subsection{Lipschitz Properties of Several Activation Function}
In this section, our goal is to prove three types of activation functions that are widely used in complex-valued neural networks are lipschitz continuous. 

The first one is from \citet{nitta2002redundancy},
\begin{equation*}
\sigma_1(z)=tanh(Re(z))+itanh(Im(z)).
\end{equation*}
This activation function applies hyperbolic tangent function on both the real part and the imaginary part. Since the derivative of the hyperbolic tangent function is upper bounded by 1, hence we can see that $\sigma_1$ is 1-lipschitz in each coordinate, if we view the real part and imaginary part as different coordinates. Then we have 
\begin{align*}
 &\left\|\sigma_{1}\left(z_{1}\right)-\sigma_{1}\left(z_{1}^{\prime}\right)\right\|_{p}\nonumber\\ =&\left[\left[\tanh \left(\operatorname{Re}\left(z_{1}\right)\right)-\tanh \left(\operatorname{Re}\left(z_{1}^{\prime}\right)\right)\right]^{p} +\left[\tanh \left(\operatorname{Im}\left(z_{1}\right)\right)-\tanh \left(\operatorname{Im}\left(z_{1}^{\prime}\right)\right)\right]^{p }\right]^{\frac{1}{p}} \\
 \leq&\left[\left(\operatorname{Re}\left(z_{1}\right)-\operatorname{Re}\left(z_{1}^{\prime}\right)\right)^{p}+\left(\operatorname{Im}\left(z_{1}\right)-\operatorname{Im}\left(z_{1}^{\prime}\right)\right)^{p}\right]^{\frac{1}{p}} \\
=&\left\|z_{i}-z_{i}^{\prime}\right\|_{p}.
\end{align*}
The first inequality holds because the hyperbolic tangent function is 1-lipschitz. 

The second type of activation function is the $\mathbb{C}ReLU$ function from \citet{trabe},
\begin{equation*}\sigma_2(z)=ReLU(Re(z))+iReLU(Im(z)).\end{equation*}
This function also operates separately on both the real part and the imaginary part. The proof process is quite similar as the first one, since the ReLU function is also 1-lipschitz.
\begin{align*}
&\left\|\sigma_{1}\left(z_{1}\right)-\sigma_{1}\left(z_{1}^{\prime}\right)\right\|_{p}\nonumber\\ =&\left[\left[\tanh \left(\operatorname{Re}\left(z_{1}\right)\right)-\tanh \left(\operatorname{Re}\left(z_{1}^{\prime}\right)\right)\right]^{p}+\left[\tanh \left(\operatorname{Im}\left(z_{1}\right)\right)-\tanh \left(\operatorname{Im}\left(z_{1}^{\prime}\right)\right)\right]^{p}\right]^{\frac{1}{p}} \nonumber\\
\leq& \left[\left(\operatorname{Re}\left(z_{1}\right)-\operatorname{Re}\left(z_{1}^{\prime}\right)\right)^{p}+\left(\operatorname{Im}\left(z_{1}\right)-\operatorname{Im}\left(z_{1}^{\prime}\right)\right)^{p}\right]^{\frac{1}{p}} \\
=&\left\|z_{i}-z_{i}^{\prime}\right\|_{p}.
\end{align*}                                

The third type of activation function is 
\begin{equation*}
\sigma_3(z)=tanh(|z|)\exp(i\theta),
\end{equation*} where $\theta=arg(z)$. 
If we write $z=x+yi$, and use vector notation to represent the real part and imaginary part of the operation, we will get
\begin{equation*}
\begin{bmatrix}
Re(\sigma_3(z))\\
Im(\sigma_3(z))
\end{bmatrix}=
\begin{bmatrix}
tanh[(x^2+y^2)^\frac{1}{2}]\frac{x}{(x^2+y^2)^\frac{1}{2}}\\
tanh[(x^2+y^2)^\frac{1}{2}]\frac{y}{(x^2+y^2)^\frac{1}{2}}\\
\end{bmatrix}.
\end{equation*}

Notice that 

    \begin{equation*}
    \left|tanh[(x^2+y^2)^\frac{1}{2}]\frac{1}{(x^2+y^2)^\frac{1}{2}}\right| \leq 1.
    \end{equation*}

Hence, we have the following inequality
\begin{align*}
 &\left|tanh[(x_1^2+y_1^2)^\frac{1}{2}]\frac{x_1}{(x_1^2+y_1^2)^\frac{1}{2}}- tanh[(x_2^2+y_2^2)^\frac{1}{2}]\frac{x_2}{(x_2^2+y_2^2)^\frac{1}{2}}\right|\\
 \leq&\left|tanh[(x_1^2+y_1^2)^\frac{1}{2}]\frac{x_1}{(x_1^2+y_1^2)^\frac{1}{2}}- tanh[(x_1^2+y_1^2)^\frac{1}{2}]\frac{x_2}{(x_1^2+y_1^2)^\frac{1}{2}}\right|\nonumber\\
 &+ \left|tanh[(x_1^2+y_1^2)^\frac{1}{2}]\frac{x_2}{(x_1^2+y_1^2)^\frac{1}{2}}- tanh[(x_2^2+y_2^2)^\frac{1}{2}]\frac{x_2}{(x_2^2+y_2^2)^\frac{1}{2}}\right|\\
 =& \left| \frac{tanh[(x_1^2+y_1^2)^\frac{1}{2}]}{(x_1^2+y_1^2)^\frac{1}{2}}\right|\left|x_1-x_2\right|\nonumber\\
 &+\left| \frac{tanh[(x_1^2+y_1^2)^\frac{1}{2}]}{(x_1^2+y_1^2)^\frac{1}{2}}-\frac{tanh[(x_2^2+y_2^2)^\frac{1}{2}]}{(x_2^2+y_2^2)^\frac{1}{2}}\right||x_2|.
 \end{align*}
 Noted that, assume $g(x)=\frac{tanh(x)}{x}$, then through calculation, we have $\left|g'(x)\right|$ bounded by 1. Hence, $g(x)$ is 1-lipschitz.
 
 Therefore, we have 
\begin{align*}
   &\left| \frac{tanh[(x_1^2+y_1^2)^\frac{1}{2}]}{(x_1^2+y_1^2)^\frac{1}{2}}-\frac{tanh[(x_2^2+y_2^2)^\frac{1}{2}]}{(x_2^2+y_2^2)^\frac{1}{2}}\right||x_2|
   \nonumber\\
\leq &\left|(x_1^2+y_1^2)^\frac{1}{2}-(x_2^2+y_2^2)^\frac{1}{2}\right||x_2|\\
\leq &\left|\frac{x_1^2-x_2^2+y_1^2-y_2^2}{(x_1^2+y_1^2)^\frac{1}{2}+(x_2^2+y_2^2)^\frac{1}{2}}\right||x_2|\\
\leq &|x_2|\left|\frac{(x_1+x_2)}{(x_1^2+y_1^2)^\frac{1}{2}+(x_2^2+y_2^2)^\frac{1}{2}}\right| |x_1-x_2|\nonumber\\
+&|x_2|\left|\frac{(y_1+y_2)}{(x_1^2+y_1^2)^\frac{1}{2}
+(x_2^2+y_2^2)^\frac{1}{2}}\right| |y_1-y_2|\\
\leq &\alpha|x_1-x_2|+\alpha|y_1-y_1|
\end{align*}
for some constant $\alpha$ such that $|x_2|\leq \alpha$.

Then, we can bound the first coordinate by 
\begin{align*} 
&\left|tanh[(x_1^2+y_1^2)^\frac{1}{2}]\frac{x_1}{(x_1^2+y_1^2)^\frac{1}{2}}- tanh[(x_2^2+y_2^2)^\frac{1}{2}]\frac{x_2}{(x_2^2+y_2^2)^\frac{1}{2}}\right|\nonumber\\
\leq &(\alpha+1)|x_1-x_2|+\alpha|y_1-y_1|.
\end{align*}
Without loss of generality, the second coordinate is bounded by \begin{equation*}
\alpha|x_1-x_2|+(\alpha+1)|y_1-y_1|.
\end{equation*} 

Finally, we have 
\begin{align*}
&\left\|\sigma_3(z_1)-\sigma_3(z_2)\right\|_p \nonumber\\
\leq &\left(\left((\alpha+1)|x_1-x_2|+\alpha|y_1-y_2|\right)^p+\left(\alpha|x_1-x_2|+(\alpha+1)|y_1-y_2|\right)^p\right)^\frac{1}{p}\\
\leq &\left(M|x_1-x_2|^p+M|y_1-y_2|^p\right)^\frac{1}{p}=(2\alpha+1)\left\|z_1-z_2\right\|_p,
\end{align*}
where $M=(2\alpha+1)^p$, $z_1=x_1+iy_1$ and $z_2=x_2+iy_2$. 

Hence, we have proved that $\sigma_3(z)=tanh(|z|)exp(i\theta)$ is lipschitz continuous.

\subsection{Proof of Theorem 1}
\subsubsection{Proof of Lemma 1}
Before proving lemma 1, we first introduce the Maurey's sparsification lemma \citep{pisie,bartl}. 
\begin{lemma}[Maurey's sparsification lemma \citep{pisie}] In a Hilbert space $\mathcal{H}$ equipped with norm $||\cdot||$, consider $f\in \mathcal{H}$ such that $f=\sum\limits_{i=1}^{n}\alpha_i g_i$ where $g_i\in \mathcal{H}$, $\alpha_i$ are positive real numbers, and $\alpha=\sum\limits_{i=1}^{n}\alpha_i\neq 0$. Then, for any positive integer k, there always exist non negative integers $k_1, k_2, ..., k_n$ such that $\sum\limits_{i=1}^{n}k_i=k$ such that  
\begin{equation*}\left\|f-\frac{\alpha}{k} \sum_{i=1}^{n} k_{i} g_{i}\right\|^{2} \leq \frac{\alpha}{k} \sum_{i=1}^{n} \alpha_{i}\left\|g_{i}\right\|^{2} \leq \frac{\alpha^{2}}{k} \max _{i}\left\|g_{i}\right\|^{2}
\end{equation*}
i.e.
\begin{equation*}\left\|\sum_{i=1}^{n}\frac{\alpha_i}{\alpha}g_i-\sum_{i=1}^{n}\frac{  k_{i} }{k}g_{i} \right\|^{2} \leq \sum_{i=1}^{n}\frac{ \alpha_{i}}{k}\left\|g_{i}\right\|^{2}  \leq \frac{\alpha}{k} \max _{i}\left\|g_{i}\right\|^{2}.
\end{equation*}

\end{lemma}

\begin{proof}
Define k i.i.d random variable $W_1,W_2,...,W_k$ such that $P(W_1=\alpha g_i)=\frac{\alpha_i}{\alpha}$. Let $W=\frac{\sum_{i=1}^{k}W_i}{k}$. Therefore,
\begin{equation*}
\mathop{E}[W]=\mathop{E}[W_1]=f.
\end{equation*}
Hence, we have 

\begin{align*}
    \mathop{E}[\left\|f-W\right\|^2]&=\frac{1}{k^2}E[\langle\sum_{i=1}^{k}(f-W_i),\sum_{i=1}^{k}(f-W_i)\rangle]\\
&=\frac{1}{k^2}\mathop{E}[\sum_{i=1}^{k}\left\|f-W_i\right\|^2]\\
&=\frac{1}{k}\mathop{E}[\left\|f-W_1\right\|^2] \\
&=\frac{1}{k}(\mathop{E}[\left\|W_1\right\|^2]-\left\|f\right\|^2)\\
&\leq\frac{1}{k} \mathop{E}[\left\|W_1\right\|^2]\\
&=\sum_{i=1}^{n}\frac{\alpha_i}{k\alpha}\cdot \alpha^2\left\|g_i\right\|^2\\
&\leq \frac{\alpha^2}{k}\max_{i}\left\|g_i\right\|^2.
\end{align*}

Since for a random variable, the minimal value it can take is at most the value of expectation, hence, there must exist a sequence of k numbers $(l_1, l_2,...,l_k)\in \{1,2,...,n\}^k$, such that $W_i=\alpha g_{l_i}$, $W=\sum_{i=i}^{k}W_i$, and 
\[\left\|W-f\right\|^2\leq  \frac{\alpha^2}{k}\max_{i}\left\|g_i\right\|^2.\]
To finish the proof, we assign integer $k_i$ mentioned in the lemma to be 

\begin{equation*}k_i=\sum_{j=1}^{k}\mathbbm{1}_{g_{l_j}=i}.\end{equation*}

\end{proof}

As \citet{bartl} indicated, the Maurey sparsification lemma only discussed the $L_1 $ norm case.  \citet{zhang} generalized this lemma to create bounds for non-$L_1$ norm cases, which is also applicable in our proof of lemma 1. 

\begin{proof}(lemma 1)\\
Given the data matrix $Z\in \mathbb{C}^{n\times d}$, re-scaling each column of the matrix $Z$ and get a matrix $Y\in \mathbb{C}^{n\times d}$, where
\begin{equation*}Y_{:;j}=Z_{:;j}/\left\|Z_{:;j}\right\|.\end{equation*}
Set $N=4dm$,  $k=\left\lceil a^{2} b^{2} m^{2 / r} / \epsilon^{2}\right\rceil$, and $\alpha=a m^{1 / r}\|X\|_{p}$. To construct an appropriate convex hull, we define 
\begin{align*}
    \{V_1, V_2, ..., V_N\}&=\left\{\sigma Y \mathbf{e}_i \mathbf{e}_j^T: \sigma \in \left\{-1, +1\right\}, i\in\left\{1,2,...,d\right\},j\in\left\{i,2,...,m\right\}\right\}  \nonumber \\
    &\cup \left\{\sigma Y \mathbf{c}_i\mathbf{e}_j^T:\sigma \in \left\{-1, +1\right\}, i\in\left\{1,2,...,d\right\},j\in\left\{i,2,...,m\right\} \right\}.\nonumber
    \end{align*}
and

\begin{align*}
C &=\left\{\frac{\alpha}{k} \sum_{i=1}^{N} k_{i} V_{i}: k_{i} \geq 0, \sum_{i=1}^{N} k_{i}=k\right\} \\
&=\left\{\frac{\alpha}{k} \sum_{m=1}^{k} V_{l_{m}}:\left(l_{1}, \ldots, l_{m}\right) \in[N]^{k}\right\},\nonumber
\end{align*}
where $k_i\mathop{=}\limits^{\Delta}\sum_{m=1}^{k}\mathbbm{1}_{l_m=i}$. 

Here, $\mathbf{e}_i$ defines the d-dimensional standard vector, $\mathbf{e}_j$ defines the m-dimensional standard vector, and $\mathbf{c}_i$ defines the d-dimensional vector in which only the ith entry equals  $\sqrt{-1}$, and other entries equal 0. 

Because of the way $V_i$ defined and $p\leq 2$, we have 
\begin{equation*}
\max _{i}\left\|V_{i}\right\|_{2} \leq \max _{i}\left\{\left\|Y \mathbf{e}_{i}\right\|_{2},\left\|Y\mathbf{c}_i\right\|_2\right\}=\max_i\left\{\left\|Y\mathbf{e}_i\right\|_2\right\}=\max _{i} \frac{\left\|X \mathbf{e}_{i}\right\|_{2}}{\left\|X \mathbf{e}_{i}\right\|_{p}} \leq 1.\nonumber
\end{equation*}

The first equality is due to the definition of complex-valued vector norms, and the second equality holds because of the monotonicity of matrix norm in terms of p. 

Next, it suffices to prove that $C$ is a cover of $\left\{Z A: A \in \mathbb{C}^{d \times m},\|A\|_{q, s} \leq a\right\}$. To prove this, we desire to bound $\left\|Z A-\frac{\alpha}{k} \sum_{i=1}^{N} k_{i} V_{i}\right\|_{2}$ by $\epsilon$ for some $(k_1,...,k_N)$.

Define $M\in \mathbb{R}^{d\times m}$ where the element of each row j equals $\left\|Z_{:;j}\right\|_p$, hence we have \begin{equation*}Z A=Y(M \odot A)\end{equation*}, where $\odot$ represents the Hadamard  product.
\begin{align*}
\|M\|_{p, r} 
&=\left\|\left(\left\|\left(\left\|Z_{:, 1}\right\|_{p}, \ldots,\left\|Z_{:, d}\right\|_{p}\right)\right\|_{p}, \ldots,\left\|\left(\left\|Z_{:, 1}\right\|_{p}, \ldots,\left\|Z_{:, d}\right\|_{p}\right)\right\|_{p}\right)\right\|_{r} \\
&=m^{1 / r}\left\|\left(\left\|Z_{:, 1}\right\|_{p}, \ldots,\left\|Z_{:, d}\right\|_{p}\right)\right\|_{p}=m^{1 / r}\left(\sum_{j=1}^{d}\left\|Z_{:, j}\right\|_{p}^{p}\right)^{1 / p} \\
&=m^{1 / r}\left(\sum_{j=1}^{d} \sum_{i=1}^{n} Z_{i, j}^{p}\right)^{1 / p}\\
&=m^{1 / r}\|Z\|_{p}.
\end{align*}
Hence, if we denote $S=M\odot A$, we have
\begin{equation*}
\|S\|_{1} \leq\langle M ,|A|\rangle \leq\|M\|_{p, r}\|A\|_{q, s} \leq m^{1 / r}\|Z\|_{p} a=\alpha.
\end{equation*}
We can see that $ZA$ indeed lies in a convex hull related with $\left\{V_1,V_2,...,V_N\right\}$:
 \begin{align*}
Z A &=Y M \\
&=Y \sum_{i=1}^{d} \sum_{j=1}^{m}\left(\operatorname{Re}\left(M_{i, j}\right) \mathbf{e}_{i} \mathbf{e}_{j}^{\top}+\operatorname{Im}\left(M_{i, j}\right) \mathbf{c}_{i} \mathbf{e}_{j}^{T}\right) \\
&=\|M\|_{1} \sum_{i=1}^{d} \sum_{j=1}^{m}\left(\frac{\operatorname{Re}\left(M_{i j}\right)}{\|M\|_{1}}\left(Y \mathbf{e}_{i} \mathbf{e}_{j}^{\top}\right)+\frac{\operatorname{Im}\left(M_{i j}\right)}{\|M\|_{1}}\left(Y \mathbf{c}_{i} \mathbf{e}_{j}^{\top}\right)\right) \\
& \in \alpha \cdot \operatorname{conv}\left(\left\{V_{1}, \ldots, V_{N}\right\}\right).
\end{align*}
where $\operatorname{conv}\left(\left\{V_{1}, \ldots, V_{N}\right\}\right)$ denotes the convex hull formed by $\{V_1, V_2, ..., V_N\}$.

Finally, by Lemma 5, there exist non-negative integers $(k_1,k_2,...,k_N)$ such that 
\begin{align*}
      \left\|Z A-\frac{\alpha}{k} \sum_{i=1}^{N}k_{i}V_{i}\right\|_{2}^{2}&=\left\|YM-\frac{\alpha}{k} \sum_{i=1}^{N} k_{i} V_{i}\right\|_{2}^{2}\\
      &\leq \frac{\alpha^{2}}{k} \max _{i}\left\|V_{i}\right\|_{2}\\
      &\leq \frac{a^{2} m^{2 / r}\|Z\|_{p}^{2}}{k}\\
      &\leq \epsilon^{2}.
\end{align*}
Hence, $C$ is a covering of the desire set. Since the cardinality of set C equals $N^k$, we have the target inequality:
\begin{equation*}\ln \mathcal{N}\left(\left\{Z A: A \in \mathbb{C}^{d \times m},\|A\|_{q, s} \leq a\right\}, \epsilon,\|\cdot\|_{2}\right) \leq \left\lceil \frac{a^{2} b^{2} m^{2 / r}}{\epsilon^{2}}\right\rceil \ln (4 d m).\end{equation*}

\end{proof}

\subsubsection{Proof of Lemma 2}
As stated in the third section, this lemma shall be proved by mathematical induction. The basic idea is as following. Denotes $Z_i$ to be the data set passing from the i-1th layer to the ith layer ($Z_0=Z^T$). According to Lemma 1, assume that fixed a specific layer i, there exists a sequence of covering matrices $(\widehat{A}_0,\widehat{A}_1,...,\widehat{A}_{i-1})$ for i-1 previous layers, and a covering matrix $\widehat{A}_i$ such that $\left\|A_i\widehat{Z_i}-\widehat{A}_i\widehat{Z}_i\right\|_2\leq\epsilon$ for some $\epsilon>0$. As a consequence, the input data for the i+1th layer shall be $Z_{i+1}=\sigma_{i+1}(A_iZ_i)$, and $\widehat{Z}_{i+1}=\sigma_{i+1}(\widehat{A}_i\widehat{Z}_i)$.

\begin{align*}
\left\|Z_{i+1}-\widehat{Z}_{i+1}\right\|_{2} & \leq \rho_{i}\left\|A_{i} Z_{i}-\widehat{A}_{i} \widehat{Z}_{i}\right\|_{2} \\
& \leq \rho_{i}\left(\left\|A_{i} Z_{i}-A_{i} \widehat{Z}_{i}\right\|_{2}+\left\|A_{i} \widehat{Z}_{i}-\widehat{A}_{i} \widehat{Z}_{i}\right\|_{2}\right) \\
& \leq \rho_{i}\left\|A_{i}\right\|_{\sigma}\left\|Z_{i}-\widehat{Z}_{i}\right\|_{2}+\rho_{i} \epsilon_{i}.
\end{align*}
Since the first term of the right hand side part depends on the inductive hypothesis, hence intuitively, we can see that the covering number upper bound depends on the product of spectral norms of all covering matrices. The detailed proof is illustrated as follows.

We first define two sequences of vector space $\{V_1,V_2,...,V_L\}$, and $\{W_2,W_3,...,W_{L+1}\}$. The first sequence of vector spaces are equipped with $\left\|\cdot\right\|_V$, and the second sequences are equipped with $\left\|\cdot\right\|_W$. For each layer's input matrix, $Z_i\in V_i$, and the first layer's input $Z\in V_1$ have the constraint:$\left\|Z\right\|_V\leq B$

Moreover, under our assumptions, $A_i$ can be viewed as a linear operator: $V_i\rightarrow W_{i+1}$, and the norm of each linear operation is defined as: 
\begin{equation*}\left\|A_i\right\|_{V\rightarrow W}=\sup _{|Z|_{V} \leq 1}\left\|A_{i} Z\right\|_{W}=c_{i}.\end{equation*}
$\sigma_i$ can be treated as a mapping from $W_{i}\rightarrow V_{i}$, and the $\rho_i-$lipschitz property means
\begin{equation*}\left\|\sigma_i(z)-\sigma_i(z')\right\|_V\leq \rho_i\left\|z-z'\right\|_W.\end{equation*}
With these preparations, we claim the following lemma which is based on a similar lemma raised by \citet{bartl}. 
\begin{lemma}[\cite{bartl}]
 Assume that a sequence of positive numbers $\left(\epsilon_{1}, \ldots, \epsilon_{L}\right)$, along with Lipschitz non-linear mappings $\left(\sigma_{1}, \ldots, \sigma_{L}\right)$ (where $\sigma_{i}$ is $\rho_{i}$ - Lipschitz), and linear operator norm bounds $\left(c_{1}, \ldots, c_{L}\right) $ as described above are given. Suppose the sequence of matrices $\mathcal{A}=\left(A_{1}, \ldots, A_{L}\right)$ lies within $\mathcal{B}_{1} \times \cdots \times \mathcal{B}_{L}$ where $\mathcal{B}_{i}$ are classes satisfying the property that each $A_{i} \in \mathcal{B}_{i}$ has $\left\|A_{i}\right\|_{V\rightarrow W} \leq c_{i}$.
Let data $Z$ be given with $\left\|Z\right\|_{V} \leq B .$ Then, define $\tau:=\sum_{j \leq L} \epsilon_{j} \rho_{j} \prod_{l=j+1}^{L} \rho_{l} c_{l}$, the complex-valued neural network images $\mathcal{F}:=\left\{F_{\mathcal{A}}(Z): \mathcal{A} \in \mathcal{B}_{1} \times \cdots \times \mathcal{B}_{L}, \left\|Z\right\|_V \leq B\right\}$ has the following covering number bound
\begin{equation*}
\mathcal{N}\left(\mathcal{F}, \tau,\left\|\cdot\right\|_{V}\right) \leq \prod_{i=1}^{L} \sup _{\left(A_{1}, \ldots, A_{i-1}\right) \atop \forall j<i . A_{j} \in \mathcal{B}_{j}} \mathcal{N}\left(\left\{A_{i} F_{\left(A_{1}, \ldots, A_{i-1}\right)}(Z): A_{i} \in \mathcal{B}_{i}\right\}, \epsilon_{i},\|\cdot\|_{W}\right) .
\end{equation*}
\end{lemma}

\begin{proof}

The lemma is proved by Mathematical induction. 

A sequence of covering set $\{\mathcal{F}_1,\mathcal{F}_2,...,\mathcal{F}_L\}$ is constructed where $\mathcal{F}_i$ covers $W_i$

\textbf{Base case}:
When i=1, we have $\mathcal{F}_1$ to be constructed according to Lemma 1, and 
\begin{equation*}\left|\mathcal{F}_{1}\right| \leq \mathcal{N}\left(\left\{A_{1} Z: A_{1} \in \mathcal{B}_{1}\right\}, \epsilon_{1},\|\cdot\|_{W}\right)=: N_{1}.
\end{equation*}

\textbf{Inductive Hypothesis}: Assume that for i=n, we can find a $\epsilon_n$-covering $\mathcal{F}_n$ for set $\left\{A_{n} \mathcal{F}_{\mathcal{A}_1,...,\mathcal{A}_{n-1}}(Z): A_{n} \in \mathcal{B}_{n}\right\} $ such that:
\begin{equation*}|\mathcal{F}_n|\leq \prod_{l=1}^{n} N_{l}.\end{equation*}

\textbf{Induction Step}:
For every element $F \in \mathcal{F}_{n}$, construct an $\epsilon_{n+1}$-cover $\mathcal{G}_{n+1}(F)$ of
\begin{equation*}
\left\{A_{n+1} \sigma_{n}(F): A_{n+1} \in \mathcal{B}_{n+1}\right\}.
\end{equation*}
Since these covers are proper, meaning $F=A_{n+1} F_{\left(A_{1}, \ldots, A_{n}\right)}(Z)$ for some matrices $\left(A_{1}, \ldots, A_{n}\right) \in$ $\mathcal{B}_{1} \times \cdots \times \mathcal{B}_{n}$, it follows that
\begin{align*}
\left|\mathcal{G}_{n+1}(F)\right| &\leq \sup _{\left(A_{1}, \ldots, A_{n}\right) \atop \forall j \leq i . A_{j} \in \mathcal{B}_{j}} \mathcal{N}\left(\left\{A_{n+1} F_{A_{1}, \ldots, A_{n}}(Z): A_{n+1} \in \mathcal{B}_{n+1}\right\}, \epsilon_{n+1},\|\cdot\|_{W}\right)\\
&=: N_{n+1}.
\end{align*}

Lastly, we can form the cover
\begin{equation*}
\mathcal{F}_{n+1}:=\bigcup_{F \in \mathcal{F}_{n}} \mathcal{G}_{n+1}(F),
\end{equation*}
whose cardinality satisfies
\begin{equation*}
\left|\mathcal{F}_{n+1}\right| \leq\left|\mathcal{F}_{n}\right| \cdot N_{n+1} \leq \prod_{l=1}^{n+1} N_{l}.
\end{equation*}

Define $
\mathcal{H}:=\left\{\sigma_{L}(F): F \in \mathcal{F}_{L}\right\}$. It's trivial to see that the cardinality of $\mathcal{H}$ is the same as $\mathcal{F}_L$. Then, It suffices to show that $\mathcal{H}$ is indeed a covering of $\mathcal{F}$. If we fix any $\left(A_{1}, \ldots, A_{L}\right)$ satisfying the constraints, then recursively, we denote \begin{equation*}
F_{1}=A_{1} Z \in W_{2}, \quad G_{i}=\sigma_{i}\left(F_{i}\right) \in V_{i+1} \quad F_{i+1}=A_{i+1} G_{i} \in W_{i+2}\end{equation*}
In other words, we need to prove that there exist $\widehat{G}_{L} \in \mathcal{H}$ such that $\left\|G_{L}-\widehat{G}_{L}\right\|_{V} \leq \tau$

\textbf{Base case}: $\text { Set } \widehat{G}_{0}=Z$.

\textbf{Inductive hypothesis}: $\text { Choose } \widehat{F}_{i} \in \mathcal{F}_{i} \text { with }\left\|A_{i} \widehat{G}_{i-1}-\widehat{F}_{i}\right\|_{W} \leq \epsilon_{i}, \text { and set } \widehat{G}_{i}:=\sigma_{i}\left(\widehat{F}_{i}\right)$.

\textbf{Induction Step}:
\begin{align*}
\left\|G_{i+1}-\widehat{G}_{i+1}\right\|_{V} & \leq \rho_{i+1}\left\|F_{i+1}-\widehat{F}_{i+1}\right\|_{W} \\
& \leq \rho_{i+1}\left\|F_{i+1}-A_{i+1} \widehat{G}_{i}\right\|_{W}+\rho_{i+1}\left\|A_{i+1} \widehat{G}_{i}-\widehat{F}_{i+1} \mid \right\|_{W}\\
& \leq \rho_{i+1}\left\|A_{i+1}\right\|_{V\rightarrow W}\left\|G_{i}-\widehat{G}_{i}\right\|_{V}+\rho_{i+1} \epsilon_{i+1} \\
& \leq \rho_{i+1} c_{i+1}\left(\sum_{j \leq i} \epsilon_{j} \rho_{j} \prod_{l=j+1}^{i} \rho_{l} c_{l}\right)+\rho_{i+1} \epsilon_{i+1} \\
&=\sum_{j \leq i+1} \epsilon_{j} \rho_{j} \prod_{l=j+1}^{i+1} \rho_{l} c_{l} \\
&= \gamma.
\end{align*}
Hence, we get proved. 
\end{proof}

To prove Lemma 2, the key idea is to apply the result of Lemma 1 and Lemma 6. 

\begin{proof}(Lemma 2) 

To begin with, we assume the same setting as above. However, to prove Lemma 2, $\left\|\cdot\right\|_V=\left\|\cdot\right\|_W=\left\|\cdot\right\|_2$, and the operator norm is set to the spectral norm, i.e. $\left\|A_i\right\|_{V\rightarrow W}=\left\|A_i\right\|_{\sigma}$. Also, the sequence of number $\{\epsilon_1,\epsilon_2,...,\epsilon_L\}$ are defined as 
\begin{equation*}\epsilon_{i}:=\frac{\alpha_{i} \epsilon}{\rho_{i} \prod_{j>i} \rho_{j} s_{j}} \quad \text { where } \quad \alpha_{i}:=\frac{1}{\bar{\alpha}}\left(\frac{b_{i}}{s_{i}}\right)^{2 / 3}, \quad \bar{\alpha}:=\sum_{j=1}^{L}\left(\frac{b_{j}}{s_{j}}\right)^{2 / 3}.\end{equation*}

By this setting, we find that the $\gamma$ defined in Lemma 6 satisfies 
\begin{equation*}\tau \leq \sum_{j \leq L} \epsilon_{j} \rho_{j} \prod_{l=j+1}^{L} \rho_{l} s_{l}=\sum_{j \leq L} \alpha_{j} \epsilon=\epsilon.\end{equation*}

Then 
\begin{align*}
&\ln \mathcal{N}\left(\mathcal{F}_{\mid S}, \epsilon,\|\cdot\|_{2}\right)\\
&\leq \sum_{i=1}^{L} \sup _{\left(A_{1}, \ldots, A_{i-1}\right) \atop \forall j<i . A_{j} \in \mathcal{B}_{j}} \ln \mathcal{N}\left(\left\{A_{i} F_{\left(A_{1}, \ldots, A_{i-1}\right)}\left(Z^{\top}\right): A_{i} \in \mathcal{B}_{i}\right\}, \epsilon_{i},\|\cdot\|_{2}\right)\\
&= \sum_{i=1}^{L} \sup _{\left(A_{1}, \ldots, A_{i-1}\right) \atop \forall j<i . A_{j} \in \mathcal{B}_{j}} \ln \mathcal{N}\left(\left\{F_{\left(A_{1}, \ldots, A_{i-1}\right)}\left(Z^{\top}\right)^{\top}\left(A_{i}\right)^{\top}:\left\|A_{i}^{\top}\right\|_{2,1} \leq b_{i},\right.\right.\\
&~~~~~~~~~~~~~~~~~~~~~~~~~~~~~~~~~~~~~~~~~~~\left.\left.\left\|A_{i}\right\|_{\sigma} \leq s_{i}\right\}, \epsilon_{i},\|\cdot\|_{2}\right)\\
&\leq \sum_{i=1}^{L} \sup _{\left(A_{1}, \ldots, A_{i-1}\right) \atop \forall j<i, A_{j} \in \mathcal{B}_{j}} \ln \mathcal{N}\left(\left\{F_{\left(A_{1}, \ldots, A_{i-1}\right)}\left(Z^{\top}\right)^{\top}\left(A_{i}\right)^{\top}:\left\|A_{i}^{\top}\right\|_{2,1} \leq b_{i}\right\}, \epsilon_{i},\|\cdot\|_{2}\right)\\
&\leq \sum_{i=1}^{L} \sup _{\left(A_{1}, \ldots, A_{i-1}\right) \atop \forall j<i . A_{j} \in \mathcal{B}_{j}} \frac{b_{i}^{2}\left\|F_{\left(A_{1}, \ldots, A_{i-1}\right)}\left(Z^{\top}\right)^{\top}\right\|_{2}^{2}}{\epsilon_{i}^{2}} \ln \left(4 W^{2}\right)
\end{align*}

The first equality holds because we use $L_2$ norms here. Hence the covering number for a matrix and its transpose are the same. To further simplify the formula, we can upper bound $\left\|F_{\left(A_{1}, \ldots, A_{i-1}\right)}\left(Z^{\top}\right)^{\top}\right\|_{2}^{2}$ by 
\begin{align*}
\left\|F_{\left(A_{1}, \ldots, A_{i-1}\right)}\left(Z^{\top}\right)^{\top}\right\|_{2} &=\left\|F_{\left(A_{1}, \ldots, A_{i-1}\right)}\left(Z^{\top}\right)\right\|_{2} \\
&=\| \sigma_{i-1}\left(A_{i-1} F_{\left(A_{1}, \ldots, A_{i-2}\right)}\left(Z^{\top}\right)-\sigma_{i-1}(0) \|_{2}\right.\\
& \leq \rho_{i-1}\left\|A_{i-1} F_{\left(A_{1}, \ldots, A_{i-2}\right)}\left(X^{\top}\right)-0\right\|_{2} \\
& \leq \rho_{i-1}\left\|A_{i-1}\right\|_{\sigma}\left\|F_{\left(A_{1}, \ldots, A_{i-2}\right)}\left(Z^{\top}\right)\right\|_{2}.
\end{align*}

Inductively, we have
\begin{equation*}
\max _{j}\left\|F_{\left(A_{1}, \ldots, A_{i-1}\right)}\left(Z^{\top}\right)^{\top} \mathbf{e}_{j}\right\|_{2} \leq\|Z\|_{2} \prod_{j=1}^{i-1} \rho_{j}\left\|A_{j}\right\|_{\sigma}.
\end{equation*}

Finally, we obtain 
\begin{align*}
\ln \mathcal{N}\left(\mathcal{F}_{\mid S}, \epsilon,\|\cdot\|_{2}\right) & \leq \sum_{i=1}^{L} \sup _{\left(A_{1}, \ldots, A_{i-1}\right) \atop \forall j<i . A_{j} \in \mathcal{B}_{j}} \frac{b_{i}^{2}\|Z\|_{2}^{2} \prod_{j<i} \rho_{j}^{2}\left\|A_{j}\right\|_{\sigma}^{2}}{\epsilon_{i}^{2}} \ln \left(4 W^{2}\right) \\
& \leq \sum_{i=1}^{L} \frac{b_{i}^{2} B^{2} \prod_{j<i} \rho_{j}^{2} s_{j}^{2}}{\epsilon_{i}^{2}} \ln \left(4 W^{2}\right) \\
&=\frac{B^{2} \ln \left(4 W^{2}\right) \prod_{j=1}^{L} \rho_{j}^{2} s_{j}^{2}}{\epsilon^{2}} \sum_{i=1}^{L} \frac{b_{i}^{2}}{\alpha_{i}^{2} s_{i}^{2}} \\
&=\frac{B^{2} \ln \left(4 W^{2}\right) \prod_{j=1}^{L} \rho_{j}^{2} s_{j}^{2}}{\epsilon^{2}}\left(\bar{\alpha}^{3}\right) .
\end{align*}
\end{proof}

\subsubsection{Proof of Theorem 1}
As stated in the third section, the main theorem we used to prove Theorem 1 is the Dudley Entropy Integral. The standard Dudley Entropy Integral introduces a method to obtain Rademacher complexity bound via covering number \citep{mohri}. 
\begin{theorem}[\citep{mohri}] 

Let $\mathcal{F}$ be a real-valued function class taking values in $[0,1]$, and assume that $\mathbf{0} \in \mathcal{F}$. Then
\begin{equation*}
\mathfrak{R}\left(\mathcal{F}_{\mid S}\right) \leq \inf _{\alpha>0}\left(\frac{4 \alpha}{\sqrt{n}}+\frac{12}{n} \int_{\alpha}^{\sqrt{n}} \sqrt{\log \mathcal{N}\left(\mathcal{F}_{\mid S}, \varepsilon,\|\cdot\|_{2}\right)} d \varepsilon .\right)
\end{equation*}

\end{theorem}
\begin{proof}\citep{bartl} Let $N \in \mathbb{N}$ be arbitrary and let $\varepsilon_{i}=\sqrt{n} 2^{-(i-1)}$ for each $i \in[N]$. For each $i$ let $V_{i}$ denote the cover achieving $\mathcal{N}\left(\mathcal{F}_{\mid S}, \varepsilon_{i},\|\cdot\|_{2}\right)$, so that
\begin{equation*}
\forall f \in \mathcal{F} \quad \exists v \in V_{i} \quad\left(\sum_{t=1}^{n}\left(f\left(x_{t}\right)-v_{t}\right)^{2}\right)^{1 / 2} \leq \varepsilon_{i},
\end{equation*}
and $\left|V_{i}\right|=\mathcal{N}\left(\mathcal{F}_{\mid S}, \varepsilon_{i},\|\cdot\|_{2}\right)$. For a fixed $f \in \mathcal{F}$, let $v^{i}[f]$ denote the nearest element in $V_{i}$. Then
\begin{align*}
&\mathbb{E}_{\epsilon} \sup _{f \in \mathcal{F}} \sum_{t=1}^{n} \varepsilon_{i} f\left(x_{t}\right) \\
&=\underset{\epsilon}{\boldsymbol{\epsilon}} \sup _{f \in \mathcal{F}}\left[\sum_{t=1}^{n} \epsilon_{t}\left(f\left(x_{t}\right)-v_{t}^{N}[f]\right)+\sum_{i=1}^{N-1} \sum_{t=1}^{n} \epsilon_{t}\left(v_{t}^{i}[f]-v_{t}^{i+1}[f]\right)-\sum_{t=1}^{n} \epsilon_{t} v_{t}^{1}[f]\right] \\
&\leq \mathbb{E} \sup _{\epsilon \in \mathcal{F}}\left[\sum_{t=1}^{n} \epsilon_{t}\left(f\left(x_{t}\right)-v_{t}^{N}[f]\right)\right]+\sum_{i=1}^{N-1} \mathbb{E} \sup _{\epsilon \in \mathcal{F}}\left[\sum_{t=1}^{n} \epsilon_{t}\left(v_{t}^{i}[f]-v_{t}^{i+1}[f]\right)\right]\nonumber\\
&+\underset{\epsilon}{\mathbb{E}} \sup _{f \in \mathcal{F}}\left[\sum_{t=1}^{n} \epsilon_{t} v_{t}^{1}[f]\right].
\end{align*}

For the third term, observe that it suffices to take $V_{1}=\{\mathbf{0}\}$, which implies
\begin{equation*}
\mathbb{E} \sup _{\epsilon}\left[\sum_{f \in \mathcal{F}}\left[\sum_{t=1}^{n} \epsilon_{t} v_{t}^{1}[f]\right]=0\right].
\end{equation*}

The first term may be handled using Cauchy-Schwarz as follows:
\begin{align*}
&\mathbb{E} \sup _{\epsilon}\left[\sum_{t \in \mathcal{F}}\left[\sum_{t=1}^{n} \epsilon_{t}\left(f\left(x_{t}\right)-v_{t}^{N}[f]\right)\right]\right] \nonumber\\
\leq & \sqrt{\underset{\epsilon}{\mathbb{E}} \sum_{t=1}^{n}\left(\epsilon_{t}\right)^{2}} \sqrt{\sup _{f \in \mathcal{F}} \sum_{t=1}^{n}\left(f\left(x_{t}\right)-v_{t}^{N}[f]\right)^{2}} \\
\leq &\sqrt{n} \varepsilon_{N}
\end{align*}

Last to take care of are the terms of the form
\begin{equation*}
\mathbb{E} \sup _{\epsilon}\left[\sum_{t=1}^{n} \epsilon_{t}\left(v_{t}^{i}[f]-v_{t}^{i+1}[f]\right)\right].
\end{equation*}

For each $i$, let $W_{i}=\left\{v^{i}[f]-v^{i+1}[f] \mid f \in \mathcal{F}\right\} .$ Then $\left|W_{i}\right| \leq\left|V_{i}\right|\left|V_{i+1}\right| \leq\left|V_{i+1}\right|^{2}$,
\begin{equation*}
\mathbb{E} \sup _{\epsilon}\left[\sum_{t \in \mathcal{F}}\left[\sum_{t=1}^{n} \epsilon_{t}\left(v_{t}^{i}[f]-v_{t}^{i+1}[f]\right)\right]\right] \leq \mathbb{E} \sup _{\epsilon \in W_{i}}\left[\sum_{t=1}^{n} \epsilon_{t} w_{t}\right],
\end{equation*}
and furthermore

\begin{align*}
\sup _{w \in W_{i}} \sqrt{\sum_{t=1}^{n} w_{t}^{2}} &=\sup _{f \in \mathcal{F}}\left\|v^{i}[f]-v^{i+1}[f]\right\|_{2} \\
& \leq \sup _{f \in \mathcal{F}}\left\|v^{i}[f]-\left(f\left(x_{1}\right), \ldots, f\left(x_{n}\right)\right)\right\|_{2}\nonumber\\
+&\sup _{f \in \mathcal{F}}\left\|\left(f\left(x_{1}\right), \ldots, f\left(x_{n}\right)\right)-v^{i+1}[f]\right\|_{2} \\
& \leq \varepsilon_{i}+\varepsilon_{i+1} \\
&=3 \varepsilon_{i+1}.
\end{align*}

With this observation, the standard Massart finite class lemma \citep{mohri} implies
\begin{align*}
&\underset{\epsilon}{\mathbb{E}} \sup _{w \in W_{i}}\left[\sum_{t=1}^{n} \epsilon_{t} w_{t}\right] \nonumber\\
\leq & \sqrt{2 \sup _{w \in W_{i}} \sum_{t=1}^{n}\left(w_{t}\right)^{2} \log \left|W_{i}\right|} \leq 3 \sqrt{2 \log \left|W_{i}\right|} \varepsilon_{i+1} \leq 6 \sqrt{\log \left|V_{i+1}\right|} \varepsilon_{i+1}.
\end{align*}

Collecting all terms, this establishes

\begin{align*}
\mathbb{E}_{\epsilon} \sup _{f \in \mathcal{F}} \sum_{t=1}^{n} \epsilon_{t} f\left(x_{t}\right) & \leq \varepsilon_{N} \sqrt{n}+6 \sum_{i=1}^{N-1} \varepsilon_{i+1} \sqrt{\log \mathcal{N}\left(\mathcal{F}_{\mid S}, \varepsilon_{i+1},\|\cdot\|_{2}\right)} \\
& \leq \varepsilon_{N} \sqrt{n}+12 \sum_{i=1}^{N}\left(\varepsilon_{i}-\varepsilon_{i+1}\right) \sqrt{\log \mathcal{N}\left(\mathcal{F}_{\mid S}, \varepsilon_{i},\|\cdot\|_{2}\right)} \\
& \leq \varepsilon_{N} \sqrt{n}+12 \int_{\varepsilon_{N+1}}^{\sqrt{n}} \sqrt{\log \mathcal{N}\left(\mathcal{F}_{\mid S}, \varepsilon,\|\cdot\|_{2}\right)} d \varepsilon.
\end{align*}

Finally, select any $\alpha>0$ and take $N$ be the largest integer with $\varepsilon_{N+1}>\alpha$. Then $\varepsilon_{N}=4 \varepsilon_{N+2}<4 \alpha$, and so
\begin{align*}
&\varepsilon_{N} \sqrt{n}+12 \int_{\varepsilon_{N+1}}^{\sqrt{n}} \sqrt{\log \mathcal{N}\left(\mathcal{F}_{\mid S}, \varepsilon,\|\cdot\|_{2}\right)} d \varepsilon \nonumber\\
\leq & 4 \alpha \sqrt{n}+12 \int_{\alpha}^{\sqrt{n}} \sqrt{\log \mathcal{N}\left(\mathcal{F}_{\mid S}, \varepsilon,\|\cdot\|_{2}\right)} d \varepsilon.
\end{align*}
\end{proof}

However, it's worth noticing that $\mathcal{N}(\mathcal{F},\epsilon, \left\|\cdot\right\|_2)$ can not be directly used in Theorem 3 to obtain the upper bound of $ \hat{\mathfrak{R}}_{S}(\mathcal{G})$. Hence we raise Lemma 3 and Lemma 4 to make it applicable. We shall first prove these two lemmas.
\begin{proof}(Lemma 3)
Consider $\mathcal{H}\subset \mathcal{F}$ is a cover of family $\mathcal{F}$ which satisfies that the cardinality of $\mathcal{H}$ equals the covering number of $\mathcal{F}$. Then for any $\mathcal{F}_\mathcal{A}\in\mathcal{F}$, we have a corresponding $h\in \mathcal{H}$ such that 
\begin{equation*}
\left\|\mathcal{F}_\mathcal{A}(Z)-h(Z)\right\|_2\leq \epsilon.
\end{equation*}

Then consider $\left\|\mathcal{F}_\mathcal{A}(Z)-Y\right\|_2\in \mathcal{G}$, we have 

    \begin{align*}
    \left|\left\|\mathcal{F}_\mathcal{A}(Z)-Y\right\|_2-\left\|h(Z)-Y\right\|_2\right|&\leq  \left|\left\|\mathcal{F}_\mathcal{A}(Z)-Y-h(Z)+Y\right\|_2\right|\\
    &=\left|\left\|\mathcal{F}_\mathcal{A}(Z)-h(Z)\right\|_2\right|\\
    &\leq \epsilon.
    \end{align*}

Therefore, it's trivial that $\bar{\mathcal{H}}=\{\left\|h(Z)-Y\right\|_2:h\in \mathcal{H}\}$ is a cover of $\mathcal{G}$, and the cardinality of $\mathcal{H}$ equals that of $\bar{\mathcal{H}}$

Hence, the covering number of $\mathcal{G}$ is less than the covering number of $\mathcal{F}$
\end{proof}

\begin{proof}(Lemma 4)
Consider $\mathcal{H}\subset \mathcal{G}$ is a cover of family $\mathcal{G}$ which satisfies that the cardinality of $\mathcal{H}$ equals the covering number of $\mathcal{G}$. For any $g\in \mathcal{G}$, there exist $h\in\mathcal{H}$ such that 
\begin{equation*}\left\|\alpha g-\alpha h\right\|_2=\alpha\left\| g-h\right\|_2\leq \alpha \epsilon. \end{equation*}
Therefore, $\alpha \mathcal{H}$ is a cover of $\alpha\mathcal{G}$

Vice Versa, if  $\alpha \mathcal{H}$ is a cover of $\alpha\mathcal{G}$, then $\mathcal{H}$ is a cover of $\mathcal{G}$

Hence, Lemma 4 is get proved. 
\end{proof}
After all preparations have been done, the proof of Theorem 1 is given as follows:

\begin{proof}(Theorem 1)\\
Consider  family $\mathcal{F}=\left\{F_{\mathcal{A}}\left(z\right): \mathcal{A}=\left(A_{1}, \ldots, A_{L}\right),\left\|A_{i}\right\|_{\sigma} \leq s_{i},\left\|A_{i}^{\top}\right\|_{2,1} \leq b_{i}\right\}$, and  family 
\begin{equation*} \mathcal{G}=\left\{(z, y) \mapsto l\left(F_{\mathcal{A}}(z), y\right): F_{\mathcal{A}} \in \mathcal{F}\right\}.\end{equation*}

As a consequence of Lemma 3, 
\begin{equation*}\mathcal{N}\left(\mathcal{F}_{|S}, \epsilon,\|\cdot\|_{2}\right) \geq \mathcal{N}\left(\mathcal{G}_{|S}, \epsilon,\|\cdot\|_{2}\right)\end{equation*}
when the loss function is set to be $l(\mathcal{F}_{\mathcal{A}}(z),y)=\left\|\mathcal{F}_{\mathcal{A}}(z)-y\right\|_2$.
Since in the standard Dudley Entropy Integral, it requires the value of loss function to be always located in the interval $[0,1]$, and we make the assumption that $l(\mathcal{F}_{\mathcal{A}}(z),y)\leq M$ always holds for the given data set, hence, we can rescale the loss function by $\frac{1}{M}$. 

Define $\bar{\mathcal{G}}=\left\{(z, y) \mapsto \frac{1}{M}l\left(F_{\mathcal{A}}(z), y\right): F_{\mathcal{A}} \in \mathcal{F}\right\}$, then Lemma 4 indicates
\begin{equation*}\mathcal{N}\left(\mathcal{G}, M\epsilon,\|\cdot\|_{2}\right)=\mathcal{N}\left(\bar{ \mathcal{G}}, \epsilon,\|\cdot\|_{2}\right).\end{equation*}

Therefore

\begin{align*}
&\mathcal{N}\left(\bar{ \mathcal{G}}_{|S}, \epsilon,\|\cdot\|_{2}\right)\leq \mathcal{N}\left(\mathcal{F}_{|S}, M\epsilon,\|\cdot\|_{2}\right)\\
&\ln\mathcal{N}\left(\bar{ \mathcal{G}}_{|S}, \epsilon,\|\cdot\|_{2}\right) \leq \ln\mathcal{N}\left(\mathcal{F}_{|S}, M\epsilon,\|\cdot\|_{2}\right)\\
&\leq \frac{\|Z\|_{2}^{2} \ln \left(4 W^{2}\right)}{ M^2\epsilon^{2}}\left(\prod_{j=1}^{L} s_{j}^{2} \rho_{j}^{2}\right)\left(\sum_{i=1}^{L}\left(\frac{b_{i}}{s_{i}}\right)^{2 / 3}\right)^{3}.
\end{align*}

If we denote $R=\|Z\|_{2}^{2} \ln \left(4 W^{2}\right)\left(\prod_{j=1}^{L} s_{j}^{2} \rho_{j}^{2}\right)\left(\sum_{i=1}^{L}\left(\frac{b_{i}}{s_{i}}\right)^{2 / 3}\right)^{3}$, then we have $\ln\mathcal{N}\left(\bar{ \mathcal{G}}_{|S}, \epsilon,\|\cdot\|_{2}\right) \leq \frac{R}{M^2\epsilon^2}$.

As stated in Theorem 3, 
    \begin{align*}
    \hat{\mathfrak{R}}_{S}(\bar{\mathcal{G}}) & = \mathop{E}_{\sigma}[\sup_{\bar{g}\in \bar{\mathcal{G}}}\frac{1}{n}\sum_{i=1}^{n}\sigma_i\bar{g}(z_i)]\\
    &= \mathop{E}_{\sigma}[\sup_{g\in \mathcal{G}}\frac{1}{n}\sum_{i=1}^{n}\sigma_i\frac{1}{M}g(z_i)]\\
    &= \frac{1}{M} \hat{\mathfrak{R}}_{S}(\mathcal{G})\\
    &\leq\inf _{\alpha>0}\left(\frac{4 \alpha}{\sqrt{n}}+\frac{12}{n} \int_{\alpha}^{\sqrt{n}} \sqrt{\ln \mathcal{N}\left(\bar{\mathcal{G}}_{\mid S}, \varepsilon,\|\cdot\|_{2}\right)} d \varepsilon \right)\\
    &\leq\inf _{\alpha>0}\left(\frac{4 \alpha}{\sqrt{n}}+\frac{12}{n} \int_{\alpha}^{\sqrt{n}} \sqrt{\frac{R}{M^2\epsilon^2}} d \varepsilon \right)\\
    &=\inf _{\alpha>0}\left(\frac{4 \alpha}{\sqrt{n}}+\frac{12\sqrt{R}}{Mn } \ln \frac{\sqrt{n}}{\alpha}\right).
    \end{align*}

 To make the upper bound neater, we make a simple choice at $\alpha=\frac{1}{n}$, hence, 
 \begin{equation*}\hat{\mathfrak{R}}_{S}(\mathcal{G})\leq \frac{4M}{n^{3/2}}+\frac{18 \left\|Z\right\|_2\sqrt{2\ln(2W)}\ln{n}R_{\mathcal{A}}}{n}.\end{equation*}
 Plugging this upper bound into Theorem 2, the desired result can be obtained.

\end{proof}

\subsection{PAC Learnability of Complex-valued Neural Networks}
In this section, we desire to present the proof which shows that complex-valued neural networks are PAC-learnable.

We denote $f_S$ to be the empirical error minimizer, i.e., $f_S=\mathop{arg}\limits_{f\in \mathcal{F}}\min\frac{1}{n}\sum\limits_{i=1}^{n}l\left(f(z_i),y_i\right)$. Similarly, $f$ is the expected error minimizer: $f=\mathop{arg}\limits_{f\in \mathcal{F}}\min\mathbb{E}\left[\frac{1}{n}\sum\limits_{i=1}^{n}l\left(f(z_i),y_i\right)\right]$. $R(f)$ and $\widehat{R}(f)$ respectively represents the expected error and the empirical error. 

The concept of PAC-learnable is defined as follows.
\begin{definition}
(PAC-learnable) Let $\mathcal{F}$ be a hypothesis set. $\mathcal{A}$ is a PAC-learnable algorithm if there exists a polynomial function poly $(\cdot, \cdot, \cdot, \cdot)$ such that for any $\epsilon>0$ and $\delta>0$, for all distributions $\mathcal{D}$ over $Z$, the following holds for any sample size $m \geq \operatorname{poly}(1 / \epsilon, 1 / \delta, n$, size $(c))$ :
\begin{equation*}
\underset{S \sim \mathcal{D}^{m}}{\mathbb{P}}\left[R\left(f_{S}\right)- R(f) \leq \epsilon\right] \geq 1-\delta.
\end{equation*}
Here $f$ and $f_S$ are defined above.
\end{definition}

\begin{corollary}
Define the loss function to be $l(F_\mathcal{A}(z), y)=||F_\mathcal{A}(z)-y||_2$, and is upper-bounded by a constant M. For a complex-valued neural network: $F_{\mathcal{A}}(z):=\sigma_{L}\left(A_{L} \sigma_{L-1}\left(A_{L-1} \cdots \sigma_{1}\left(A_{1} z\right) \right)\right)$, where activation functions $\sigma_i$ are $\rho_i$-lipschitz, it is PAC-learnable. 
\end{corollary}

\begin{proof}

It suffices to prove that $R(f_S)-R(f)\leq \epsilon$ via the generalization upper bound under high probability.

Since
\begin{align*}
R(f_S)-R(f)&= R(f_S)-\widehat{R}(f_S)+\widehat{R}(f_S)-R(f)\\
&\leq R(f_S)-\widehat{R}(f_S)+\widehat{R}(f)-R(f),
\end{align*}
the last inequality holds because $f_S$ is the empirical error minimizer, therefore, we have 
    \begin{align*}
    \left|R(f_S)-R(f)\right|&\leq \left|R(f_S)-\widehat{R}(f_S)+\widehat{R}(f)-R(f)\right|\\
    &\leq \left|R(f_S)-\widehat{R}(f_S)\right|+\left|\widehat{R}(f)-R(f)\right|\\
    &\leq 2\sup\limits_{f\in\mathcal{F}}\left|\widehat{R}(f)-R(f)\right|.
    \end{align*}
    
Hence
\begin{equation*}P\left(\left|R(f_S)-R(f)\right|\leq \epsilon\right)\geq P\left(\sup\limits_{f\in\mathcal{F}}\left|\widehat{R}(f)-R(f)\right|\leq \frac{\epsilon}{2}\right)\end{equation*}.

As in Theorem 1, we have for any $f\in \mathcal{F}$

\begin{align*}
&P\left(\left|R(f)-\widehat{R}(f)\right|\leq \frac{8M}{n^\frac{3}{2}}+\frac{36||Z||_2\sqrt{2ln(2W)}ln(n)R_\mathcal{A}}{n}+3M\sqrt{\frac{ln\frac{2}{\delta}}{2n}}\right)\nonumber\\
\geq & 1- \delta.
\end{align*}

Notice that, these two statments are equivalent: 
\begin{align*}
&\sup\limits_{f\in\mathcal{F}}\left|\widehat{R}(f)-R(f)\right|\leq \frac{\epsilon}{2}\\
\Leftrightarrow &\forall f\in \mathcal F, \left|R(f)-\widehat{R}(f)\right|\leq \frac{\epsilon}{2}.
\end{align*}

Hence, we can claim that if 
\begin{equation*}\frac{\epsilon}{2}\geq\frac{8M}{n^\frac{3}{2}}+\frac{36||Z||_2\sqrt{2ln(2W)}ln(n)R_\mathcal{A}}{n}+3M\sqrt{\frac{ln\frac{2}{\delta}}{2n}},\end{equation*}
then 
\begin{equation*}P\left(\sup\limits_{f\in\mathcal{F}}\left|\widehat{R}(f)-R(f)\right|\leq \frac{\epsilon}{2}\right)\geq 1-\delta\end{equation*}
i.e.
\begin{equation*}P\left(\left|R(f_S)-R(f)\right|\leq \epsilon\right)\geq 1-\delta.\end{equation*}

Hence, we can get the conclusion that if
\begin{equation*}n\geq \frac{8}{\epsilon^3}\left(8M+36\left\|Z\right\|_2\sqrt{2\ln{(2W)}}R_{\mathcal{A}}+3M\sqrt{\frac{\ln{\frac{2}{\delta}}}{2}}\right)^3\end{equation*},
then 
\begin{equation*}P\left(\left|R(f_S)-R(f)\right|\leq \epsilon\right)\geq 1-\delta.\end{equation*}

Therefore, PAC-learnability of complex-valued neural networks get proved. 
\end{proof}

\subsection{Generalization of Sequential Data}
In this section, we aim at proving Theorem 2. Theorem 2 shows an extension of generalization to sequential data case. Therefore, sequential analogues of complexities \citep{rakhl} are presented in this section to complete the proof. 
\subsubsection{Sequential Rademacher Complexity}
In the case of classical complexity measure, we use the expectation of the supremum of Rademacher process to define the Rademacher complexity. In the sequential Rademacher case, the intuition is quite similar.  \citet{rakhl} illustrated a binary tree process to be the analogue of Rademacher process, which coincides with Rademacher process under i.i.d assumption, but behaves differently in general. The notion of a tree is defined as following:

"A $\mathcal{Z}$-valued tree $\mathbf{z}$ of depth $n$ is a rooted complete binary tree with nodes labeled by elements of $\mathcal{Z}$. We identify the tree $\mathbf{z}$ with the sequence $\left(\mathbf{z}_{1}, \ldots, \mathbf{z}_{n}\right)$ of labeling functions $\mathbf{z}_{i}:\{\pm 1\}^{i-1} \mapsto \mathcal{Z}$ which provide the labels for each node. Here, $\mathbf{z}_{1} \in \mathcal{Z}$ is the label for the root of the tree, while $\mathbf{z}_{i}$ for $i>1$ is the label of the node obtained by following the path of length $i-1$ from the root, with $+1$ indicating 'right' and $-1$ indicating 'left'. A path of length $n$ is given by the sequence $\epsilon=\left(\epsilon_{1}, \ldots, \epsilon_{n}\right) \in\{\pm 1\}^{n}$. For brevity, we shall often write $\mathbf{z}_{t}(\epsilon)$, but it is understood that $\mathbf{z}_{t}$ only depends only on the prefix $\left(\epsilon_{1}, \ldots, \epsilon_{t-1}\right)$ of $\epsilon$. Given a tree $\mathbf{z}$ and a function $f: \mathcal{Z} \mapsto \mathbb{R}$, we define the composition $f \circ \mathbf{z}$ as a real-valued tree given by the labeling functions $\left(f \circ \mathbf{z}_{1}, \ldots, f \circ \mathbf{z}_{n}\right)$." \citep{rakhl}

Therefore, the definition of sequential Rademacher complexity is stated in Definition 2.

\begin{definition}[\citep{rakhl}]
For a $\mathcal{Z}$-valued tree $\mathbf{z}$ with depth n, then the sequential Rademacher complexity of a function class

$\mathcal{G}_{sq}:=\left\{(z_t, y_t) \mapsto l\left(F_{\mathcal{A}}(z), y\right): F_{\mathcal{A}} \in \mathcal{F}\right\}$ is defined as follows:
\begin{equation*}\mathfrak{R}_{n}^{sq}(\mathcal{G}_{sq}, \mathbf{z})=\mathbb{E}\left[\sup _{f \in \mathcal{F}} \frac{1}{n} \sum_{t=1}^{n} \epsilon_{t} l\left(f\left(z_{t}(\epsilon)\right), y_t\right)\right],\end{equation*}
and 
\begin{equation*}\mathfrak{R}_{n}^{sq}(\mathcal{G}_{sq})=\sup _{\mathbf{z}} \mathfrak{R}_{n}^{sq}(\mathcal{G}_{sq}, \mathbf{z}).\end{equation*}

Here $\epsilon_{t}$ is the Rademacher variables taking value from $\{+1,-1\}$ with equal probability. 
\end{definition}

\subsubsection{Sequential Rademacher Complexity Generalization Bound}

When investigating the relation between generalization error and Rademacher complexity, we have the following theorem. 
\begin{theorem} Given function class $\mathcal{F}$, sample $S=\{(z_1,y_1), (z_2,y_2),...,(z_n,y_n)\}$ where ($z_1,y_1$) are i.i.d data points, we have 
\begin{equation*}\sup_{f\in\mathcal{F}} \mathbb{E}\left[\frac{1}{n}\sum_{i=1}^{n}f(z_i,y_i)-\mathbb{E}\left[f\right]\right]\leq \mathbb{E}\left[\sup_{f\in\mathcal{F}}\frac{1}{n}\sum_{i=1}^{n}f(z_i,y_i)-\mathbb{E}\left[f\right]\right] \leq 2\mathfrak{R}(\mathcal{F})\end{equation*}
where $\mathfrak{R}(\mathcal{F})=\mathbb{E}[\widehat{\mathfrak{R}}_S(\mathcal{F})]$

\end{theorem}

For sequential Rademacher complexity,  \citet{rakhl} proved a similar theorem. 

\begin{theorem}
Given function class $\mathcal{F}$, sample $S=\{(z_1,y_1), (z_2,y_2),...,(z_n,y_n)\}$ where ($z_1,y_1$) are sequential data points, then the following inequality holds:
\begin{equation*}\mathbb{E}\left[\sup _{f \in \mathcal{F}} \frac{1}{n} \sum_{t=1}^{n}\left(\mathbb{E}\left[f\left(z_{t}, y_{t}\right) \mid \mathcal{A}_{t-1}\right]-f\left(z_{t},y_t\right)\right)\right] \leq 2 \mathfrak{R}_{n}^{sq}(\mathcal{F})\end{equation*}
where $\mathfrak{R}_{n}^{sq}(\mathcal{F})$ denotes the sequential Rademacher complexity. 
\end{theorem}
If the function class $\mathcal{F}$ is bounded, i.e. for any $f\in \mathcal{F}$, $\left\|f\right\|_\infty\leq M$, then the generalization error $ \frac{1}{n} \sum_{t=1}^{n}\left(\mathbb{E}\left[f\left(z_{t}, y_{t}\right) \mid \mathcal{A}_{t-1}\right]-f\left(z_{t},y_t\right)\right)$ is sharply concentrated around its expectation. which leads to Corollary 1. 
\begin{corollary}
Assume that for the target function class, any $f\in \mathcal{F}$, we have$\left\|f\right\|_\infty\leq M$. Given sample $S=\{(z_1,y_1), (z_2,y_2),...,(z_n,y_n)\}$ where ($z_1,y_1$) are sequential data points, then under probability at least $1-\delta$ the following inequality holds:
\begin{equation*}\frac{1}{n} \sum_{t=1}^{n}\left(\mathbb{E}\left[f\left(z_{t}, y_{t}\right) \mid \mathcal{A}_{t-1}\right]-f\left(z_{t},y_t\right)\right)\leq 2\mathfrak{R}_{n}^{sq}(\mathcal{F})+M\sqrt{\frac{\log \frac{2}{\delta}}{2n}}.\end{equation*}
   
\end{corollary}
\begin{proof}
This corollary is a consequence of McDiarmid's Inequality and Theorem 6. By McDiarmid's Inequality, since $\left\|f\right\|_\infty\leq M$, we have 
\begin{equation*}\mathbb{P}(|\Delta(\mathcal{F})-\mathbb{E}[\Delta(\mathcal{F})]| \geq t) \leq 2 \exp \left(-2n t^{2} / M^{2}\right)\end{equation*}
where $\Delta(\mathcal{F})=\frac{1}{n} \sum_{t=1}^{n}\left(\mathbb{E}\left[f\left(z_{t}, y_{t}\right) \mid \mathcal{A}_{t-1}\right]-f\left(z_{t},y_t\right)\right)$
Then by Theorem 6, we can get the sequential Rademacher complexity upper bound. 
\end{proof}
As a consequence of Corollary 1, it's necessary to bound the sequential Rademacher complexity if we want to prove the generalization upper bound. This leads to the introduction of sequential Dudley Entropy Integral.  

\subsubsection{Sequential Dudley Entropy Integral}

Before stating the sequential Dudley Entropy Integral, we first present the definition of sequential covering number \citep{rakhl}

\begin{definition} (Sequential Covering Number)
A set $C$ is a sequential  $\alpha$-cover (with respect to $\ell_{p}$-norm) of $\mathcal{F} \subseteq \mathbb{R}^{\mathcal{Z}}$ on a tree $\mathbf{z}$ of depth $n$ if
\begin{equation*}
\forall f \in \mathcal{F}, \forall \epsilon \in\{\pm 1\}^{n}, \exists \mathbf{c} \in C \quad \text { s.t. } \quad\left(\frac{1}{n} \sum_{t=1}^{n}\left|\mathbf{c}_{t}(\epsilon)-f\left(\mathbf{z}_{t}(\epsilon)\right)\right|^{p}\right)^{1 / p} \leq \alpha.
\end{equation*}
The sequential covering number of a function class $\mathcal{F}$ on a given tree $\mathbf{z}$ is defined as
\begin{equation*}
\mathcal{N}_{p}^{sq}(\alpha, \mathcal{F}, \mathbf{z})=\min \left\{|C|: C \text { is an } \alpha \text {-cover w.r.t. } \ell_{p} \text {-norm of } \mathcal{F} \text { on } \mathbf{z}\right\} \text {. }
\end{equation*}
and define  $\mathcal{N}_{p}^{sq}(\alpha, \mathcal{F}, n)=\sup _{\mathbf{z}} \mathcal{N}_{p}^{sq}(\alpha, \mathcal{F}, \mathbf{z})$.

\end{definition}

 \citet{rakhl} provides the sequential version Dudley Entropy Integral as following: 
\begin{theorem}(Sequential Dudley Entropy Integral)
For $p \geq 2$, the sequential Rademacher complexity of a function class $\mathcal{F} \subseteq[-1,1]^{\mathcal{Z}}$ on a $\mathcal{Z}$-valued tree of depth $n$ satisfies
\begin{equation*}
\mathfrak{R}_{n}^{sq}(\mathcal{F})\leq\inf _{\alpha}\left\{4 \alpha+\frac{12}{\sqrt{n}} \int_{\alpha}^{1} \sqrt{\log \mathcal{N}_{2}^{sq}(\delta, \mathcal{F}, n)} d \delta\right\}.
\end{equation*}
\end{theorem}

Notice that for the classical $\alpha$-cover of $\mathcal{F}$ with regard to $l_2$ norm, denote it by $V$, we have for any given data matrix $Z$, and for any $\mathcal{F}_{\mathcal{A}}\in \mathcal{F}$, there exist $v\in V$ such that 
\[\left\|\mathcal{F}_\mathcal{A}(Z)-v(Z)\right\|_2\leq \alpha\]
Since given a set of sequential data, $\left\|\mathcal{F}_\mathcal{A}(Z)-v(Z)\right\|_2=\sqrt{\sum\limits_{t=1}^{n}\left(\mathcal{F}_\mathcal{A}(z_t)-v(z_t)\right)^2}$. 

Hence, if
\begin{equation*}\left\|\mathcal{F}_\mathcal{A}(Z)-v(Z)\right\|_2\leq \alpha, \end{equation*} then we have 
\begin{equation*}\sqrt{\frac{1}{n}\sum\limits_{t=1}^{n}\left(\mathcal{F}_\mathcal{A}(z_t)-v(z_t)\right)^2}\leq \frac{\alpha}{\sqrt{n}}.\end{equation*}

Hence, as a consequence of Lemma 2, 
\begin{equation*}\ln\mathcal{N}_{2}^{sq}(\frac{\alpha}{\sqrt{n}},\mathcal{G}_{sq}, n)\leq \ln \mathcal{N}(\mathcal{G}, \alpha, \left\|\cdot\right\|_2)\leq \frac{\left\|Z\right\|_2\ln{4W^2}}{\alpha^2} R_{\mathcal{A}}^{sq},\end{equation*}
where $\mathcal{G}_{sq}$ denotes the loss function family of the sequential data set, $R_{\mathcal{A}}^{sq}$ denotes the spectral complexity of the CVNNs under the case of sequential data set, and $\mathcal{G}$ denotes the loss function family of the i.i.d data set.

Hence, we have
\begin{equation*}\ln \mathcal{N}_{2}^{sq}\left(\alpha, \mathcal{G}_{s q}, n\right)\leq \frac{\|Z\|_{2} \ln 4 W^{2}}{n\alpha^{2}} R_{\mathcal{A}}^{sq}.\end{equation*}

\subsubsection{Proof of Theorem 2}

As a consequence of the previous sessions D.1-D.3, we have 
\begin{equation*}\ln \mathcal{N}_{2}^{s q}\left(\epsilon, \mathcal{G}_{s q}, n\right) \leq \frac{\|Z\|_{2} \ln 4 W^{2}}{n \epsilon^{2}} R_{\mathcal{A}}^{sq}\end{equation*}
for any $\epsilon \in \mathbbm{R}^{+}$. 

Therefore, by Lemma 4 and the sequential Dudley Entropy Integral, we can derive the following bound for the sequential Rademacher complexity: 
\begin{equation*}\mathfrak{R}_n^{sq}(\mathcal{G}_{sq}) \leq \frac{4M}{n}+12\frac{\ln{n}\sqrt{R_{\mathcal{A}}}}{Mn}\end{equation*}
where $M$ denotes the upper bound for the loss function. 

After plugging the above inequality into Corollary 2, we can get the desired bound stated in Theorem 2. 

\subsection{Additional Experiments Details}

The section provides all the additional details of our experiments.

\subsubsection{Datasets}

Our experiments are conducted on six datasets: MNIST \citep{lecun1998gradient}, FashionMNIST \citep{xiao2017/online},  CIFAR-10, CIFAR-100, \citep{krizhevsky2009learning}, IMDB \citep{maas-EtAl:2011:ACL-HLT2011}, and Tiny ImageNet \citep{le2015tiny}.
The details of these datasets are shown as follows.
\begin{itemize}
    \item 
\textbf{MNIST} consists of $60,000$ training images and $10,000$ test images from $10$ different classes. It can be downloaded from 
\url{http://yann.lecun.com/exdb/mnist/}.

\item
\textbf{FashionMNIST} consists of $60,000$ training images and $10,000$ test images from $10$ different classes. It can be downloaded from \url{https://github.com/zalandoresearch/fashion-mnist}.

\item
\textbf{CIFAR-10} consists of $50,000$ training images and $10,000$ test images from $10$ different classes, and \textbf{CIFAR-100} has the same data as CIFAR-10 while images in CIFAR-100 belong to $100$ classes. CIFAR-10 and CIFAR-100 can be downloaded from 
\url{https://www.cs.toronto.edu/~kriz/cifar.html}.

\item
\textbf{IMDB} is a movie reviews sentiment classification dataset, in which each of training and test sets consists of $25,000$ movie reviews from $2$ different classes.
It can be downloaded from 
\url{http://ai.stanford.edu/~amaas/data/sentiment/}.

\item
\textbf{Tiny ImageNet} consists of $100,000$ training images and $10,000$ test images from $200$ different classes. It can be downloaded from 
\url{http://cs231n.stanford.edu/tiny-imagenet-200.zip}.

\end{itemize}

For the image datasets, {\it i.e.}, MNIST, FashionMNIST, CIFAR-10, CIFAR-100, and Tiny ImageNet, we normalize each pixel of the images from the datasets to the range of $[0,1]$ before feeding them into the neural network.
For the IMDB dataset, we perform data pre-processing following \url{https://github.com/manavgakhar/imdbsentiment/blob/master/IMdB_sentiment_analysis_project.ipynb}.

\subsubsection{Model Architectures}
\renewcommand{\tablename}{Supplementary Table}

We employ the python package \textbf{complexPyTorch} \citep{PhysRevX.11.021060} to implement our CVNNs, which include complex-value CNNs and complex-value MLPs. The detailed architectures of CVNNs are presented in Supplementary Table \ref{table:network}, and all the parameters in these network architectures are complex values except for the last layer.

\begin{table}[ht]
\caption{Detailed model architectures for different datasets.}
\scriptsize
\centering
\vspace{2mm}
 \begin{tabular}{ccc}
\toprule
MNIST/FashionMNIST/CIFAR-10/CIFAR-100  & Tiny ImageNet & IMDB \\
\midrule
$\begin{matrix} 5\times 5,10 \\ \text{maxpool}, 2\times 2 \end{matrix}$ & $\begin{matrix} 5\times 5,10 \\ \text{maxpool}, 2\times 2 \end{matrix}$ & fc-$500$   \\
$\begin{matrix} 5\times 5,20 \\ \text{maxpool}, 2\times 2 \end{matrix}$  & $\begin{matrix} (5\times 5,20)\times 2 \\ \text{maxpool}, 2\times 2 \end{matrix}$ & fc-$200$   \\

fc-$500$ & fc-$500$ & \\

\midrule
abs & abs & abs \\
fc-$10/100$, softmax & fc-$200$, softmax & fc-$2$, softmax \\
\bottomrule
\end{tabular}
\label{table:network}
\end{table}

In Supplementary Table \ref{table:network}, "$5\times 5,10$" denotes that the convolutional layer has $5\times 5$ kernel size and $10$ output channels. The strides for all convolutional layers are setting to $1$. "fc-$500$" denotes the fully-connected layer with the output features of $500$. All convolutional layers and fully-connected layers are followed the ReLU layer except for the last layer. "abs" is the absolute layer that computes the absolute value of each element in input and can convert complex values to real values.
\subsubsection{Implementation Details}
This section provides all the additional implementation details for our experiments.

\textbf{Model training.} We employ SGD to optimize all the models with $\operatorname{momentum}=0.9$.

\textbf{Training strategy for MNIST and FashionMNIST.} Every model is trained by SGD for $100$ epochs, in which the batch size is set as $1024$, and the learning rate is fixed to $0.01$.

\textbf{Training strategy for CIFAR-10 and CIFAR-100.} Models is trained by SGD for $100$ epochs, in which the batch size is set as $128$, and the learning rate is fixed to $0.01$.

\textbf{Training strategy for IMDB.} Models is trained by SGD for $100$ epochs, in which the batch size is set as $512$. The learning rate is initialized as $0.01$ and decayed by $0.2$ every $40$ epochs.

\textbf{Training strategy for Tiny ImageNet.} Models is trained by SGD for $100$ epochs, in which the batch size is set as $128$. The learning rate is initialized as $0.01$ and decayed by $0.2$ every $40$ epochs.

\end{document}